\documentclass[nohyperref]{article}

\PassOptionsToPackage{numbers}{natbib}
\usepackage[final]{neurips_2022}

\usepackage{microtype}
\usepackage{graphicx}
\usepackage{subfigure}

\usepackage{caption}

\usepackage{float}
\usepackage{adjustbox}

\usepackage{booktabs}
\usepackage{hyperref}

\usepackage{amsmath}
\usepackage{amssymb}
\usepackage{mathtools}
\usepackage{amsthm}
\usepackage{xcolor}
\usepackage[capitalize,noabbrev]{cleveref}
\newcommand{\argmin}{\textrm{argmin}}
\newcommand{\R}{\mathbb R}
\newcommand{\E}{\mathbb E}
\newcommand{\tr}{\mathrm{tr}}

\newcommand{\Cnew}{\mathrm{C_{new}}}
\theoremstyle{plain}
\newtheorem{theorem}{Theorem}[section]
\newtheorem{proposition}[theorem]{Proposition}
\newtheorem{lemma}[theorem]{Lemma}
\newtheorem{corollary}[theorem]{Corollary}
\theoremstyle{definition}
\newtheorem{definition}[theorem]{Definition}

\theoremstyle{remark}
\newtheorem{remark}[theorem]{Remark}

\usepackage{mathabx}
\newcommand{\mse}[2]{\widebar{\mathrm{MSE}}^{#1, #2}}
\newcommand{\msei}[2]{\mathrm{MSE}^{i, #2}}

\newcommand{\Wsvd}{W_{\mathrm{SVD}}}
\newcommand{\Gsvd}{\Gamma_{\mathrm{SVD}}}
\newcommand{\thres}{\mathrm{thres}}
\usepackage[textsize=tiny]{todonotes}
\usepackage{tikz}
\usetikzlibrary{backgrounds}
\usetikzlibrary{arrows,shapes}
\usetikzlibrary{tikzmark}
\usetikzlibrary{calc}

\usepackage{amsmath}
\usepackage{amsthm}
\usepackage{amssymb}
\usepackage{mathtools, nccmath}
\usepackage{wrapfig}
\usepackage{comment}
\usepackage{blindtext}

\usepackage{graphicx}

\usepackage{xspace}

\usepackage{array}
\usepackage{ragged2e}
\newcolumntype{P}[1]{>{\RaggedRight\hspace{0pt}}p{#1}}
\newcolumntype{X}[1]{>{\RaggedRight\hspace*{0pt}}p{#1}}

\usepackage{tcolorbox}

\usepackage{tikz}
\usetikzlibrary{arrows,shapes,positioning,shadows,trees,mindmap}
\usepackage[edges]{forest}
\usetikzlibrary{arrows.meta}
\colorlet{linecol}{black!75}
\usepackage{xkcdcolors} %

\usepackage{tikz}
\usetikzlibrary{backgrounds}
\usetikzlibrary{arrows,shapes}
\usetikzlibrary{tikzmark}
\usetikzlibrary{calc}
\newcommand{\highlight}[2]{\colorbox{#1!17}{$\displaystyle #2$}}

\renewcommand{\highlight}[2]{\colorbox{#1!17}{#2}}

\newcommand{\printfnsymbol}[1]{%
  \textsuperscript{\@fnsymbol{#1}}%
}
\begin{document}

\title{Multitasking Models are Robust to Structural Failure: \\
A Neural Model for Bilingual Cognitive Reserve}
\author{%
    Giannis Daras\thanks{equal contribution.} \\
    The University of Texas at Austin \\
    giannisdaras@utexas.edu \\
    \And 
    Negin Raoof$^{*}$\\
    The University of Texas at Austin \\
    neginraoof@gmail.com \\
    \And
    Zoi Gkalitsiou \\
    The University of Texas at Austin \\
    zoi.gkalitsiou@austin.utexas.edu \\
    \And 
    Alexandros G. Dimakis \\
    The University of Texas at Austin \\
    dimakis@austin.utexas.edu \\
}
\maketitle

\begin{abstract}

We find a surprising connection between multitask learning and robustness to neuron failures. Our experiments show that bilingual language models retain higher performance under various neuron perturbations, such as random deletions, magnitude pruning and weight noise compared to equivalent monolingual ones. 
We provide a theoretical justification of this robustness by mathematically analyzing linear representation learning and showing that multitasking creates more robust representations. 
Our analysis connects robustness to spectral properties of the learned representation and proves that multitasking leads to higher robustness for diverse task vectors. We open-source our code and models in the following URL: \href{https://github.com/giannisdaras/multilingual_robustness}{https://github.com/giannisdaras/multilingual\_robustness}.

\end{abstract}

\section{Introduction}
Converging evidence from cognitive science research indicates that bilingualism increases brain robustness by reducing the rate of cognitive decline due to aging~\citep{anderson2021bilingualism,gold2013lifelong} and delaying the onset of symptoms of dementia \citep{bialystok2007bilingualism,craik2010delaying}. It appears that individuals who speak more than one language on a regular basis are able to maintain typical cognitive functioning despite neural degeneration. This mismatch between cognitive functioning and brain pathology is called Cognitive Reserve \citep{barulli2013efficiency}, and its underlying mechanisms are poorly understood and are an active topic of investigation.

Inspired by this research, we study whether \textit{artificial} neural networks are more robust when trained on multiple languages or multiple tasks. 
Our experiments demonstrate that training on multiple tasks indeed increases structural robustness. 
We train monolingual and bilingual GPT-2 models with the same architecture and dataset sizes. 
Initially, monolingual GPT-2~\citep{radford2019language} models are slightly outperforming the bilingual ones, but when we introduce structural noise (by randomly deleting neurons or adding noise to the weights) bilingual models degrade more gracefully and eventually outperform the monolingual models in the high-noise regime. For some amount of noise, bilingual models start outperforming the monolingual ones demonstrating a \textit{cross-over} in performance due to their increased robustness. We observe this phenomenon for numerous models across three different types of corruption: additive Gaussian noise to the weights, random weight pruning and magnitude-based weight pruning \citep{han}.

\begin{figure}[!t]
    \centering
    \includegraphics[width=8.5cm]{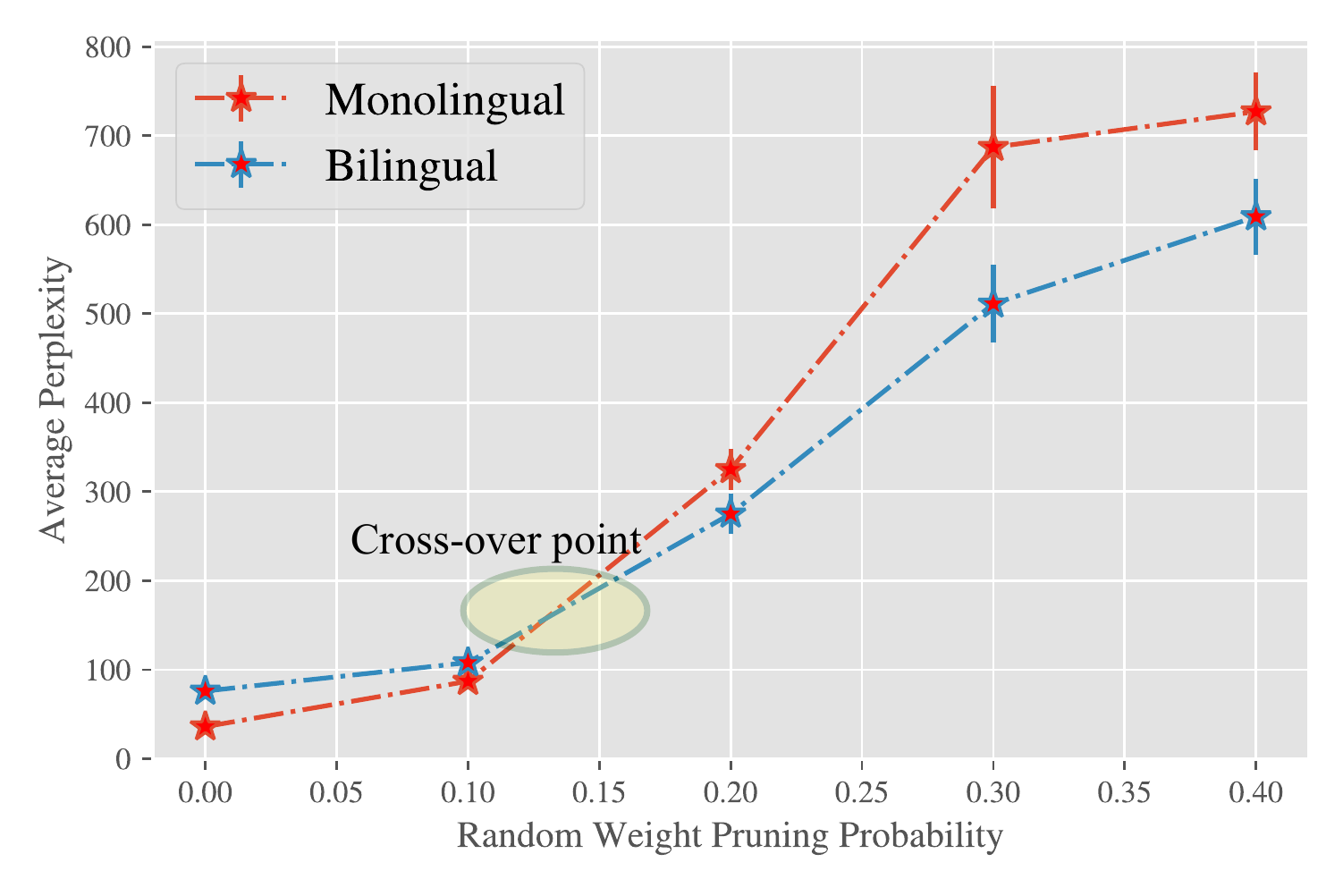}
    \caption{\small Performance of monolingual and bilingual GPT-2 models with the same architecture and training dataset size. We show the performance as we randomly erase weights. The x-axis indicates the probability of erasing an attention weight parameter (setting to it zero). The y-axis indicates the average perplexity over $20$ runs with $95$\% confidence intervals. The bilingual model initially shows slightly worse performance, but as more weights are deleted, the monolingual model declines faster and performs worse in the highly damaged regime. This indicates that the bilingual GPT-2 model is more robust to neuron weight erasures. We show similar results for several models and types of errors in our experimental section. }
   \label{fig:gpt2_random_deletion}
\end{figure}

\begin{figure}[h!]
    \centering
    \includegraphics[width=7cm]{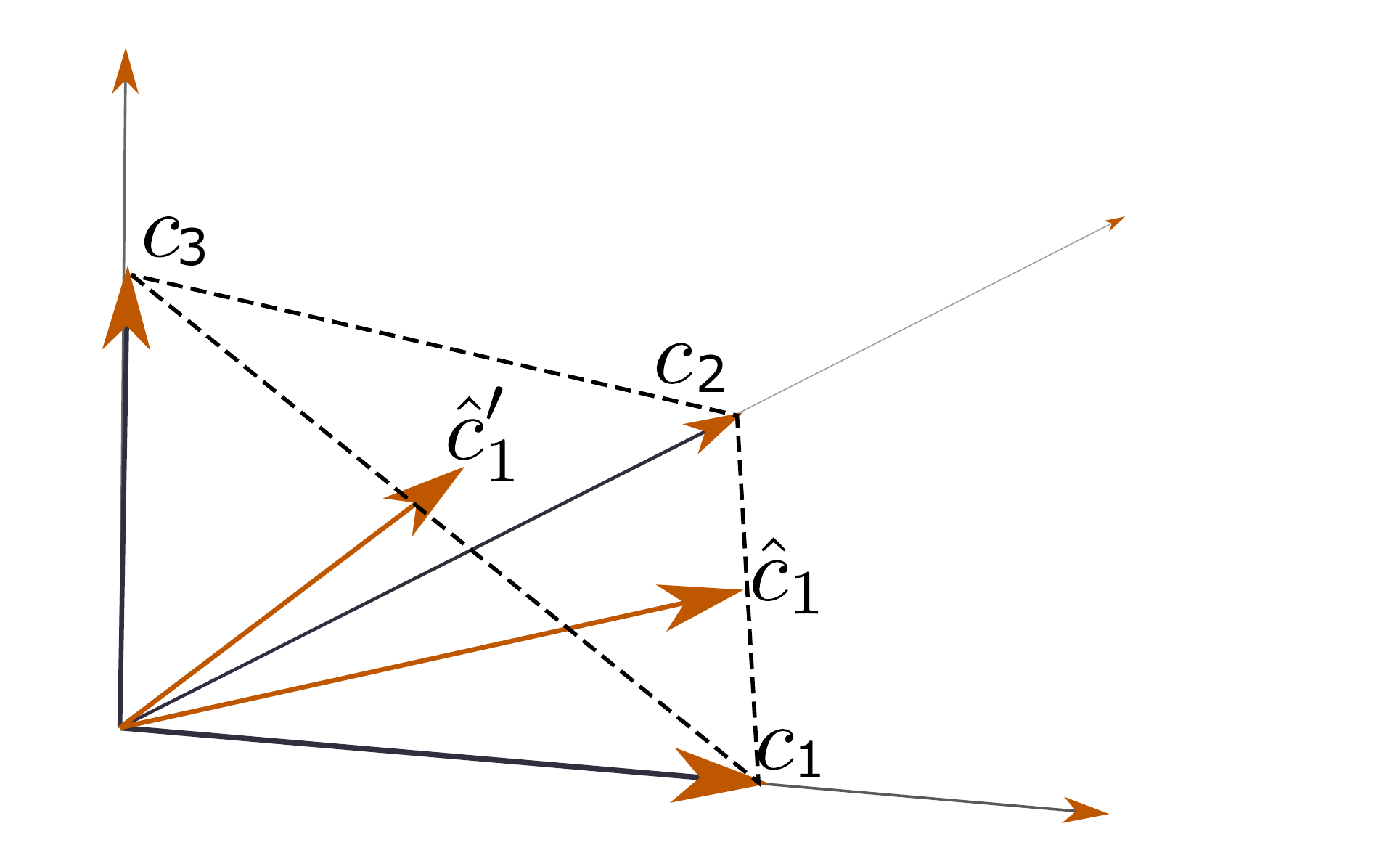}
    \caption{\small Let $c_1, c_2, c_3$ be the standard basis of $\R^3$. For two tasks, the best one dimensional approximation to $c_1,c_2$ is $\hat{c}_1= [1/2, 1/2, 0 ]^T$ but the best one dimensional approximation to three tasks $c_1,c_2,c_3$ is $\hat{c}'_1= [1/3, 1/3, 1/3 ]^T$. Multi-tasking is creating $\ell_2$ regularization since $||\hat{c}'_1||_2< ||\hat{c}_1||_2$. It is important that the original task vectors $c_1,c_2, c_3$ are orthogonal i.e. diverse, since this creates regularization.   
    }
    \label{fig:3dexample}
\end{figure}

\noindent \textbf{Our Contributions:}  
We provide a theoretical justification of this phenomenon by mathematically analyzing linear multitask representation learning~\citep{Maurer,du2020few}.  Our analysis shows that introducing more diverse tasks creates $\ell_2$ regularization in the linear task heads. Further, we formally connect the Euclidean norm of the learned representations to structural robustness under errors in the network weights. 
Our main theorem establishes that multitasking leads to higher robustness to additive noise for linear representations when the task vectors are selected as random and independent Gaussian vectors. Our results also establish that when the tasks are significantly overlapping, multitasking does not lead to higher robustness and hence task diversity is necessary.

We experimentally observe that multitasking increases structural robustness for numerous networks and multiple problems including MNIST, CIFAR10, Newsgroup20, GPT models and finetuned GPT models on GLUE tasks. We train networks under exactly comparable dataset and architecture conditions and show that models become more robust to structural failures as they are trained with more tasks. 
We experiment with three different types of structural failures and show robustness increases for all of them. 
 We also experimentally observe that the addition of diverse tasks seems to regularize the model weights, as we predict in our theoretical analysis.

\section{Theoretical Analysis}

\noindent \textbf{Building intuition.} 
We start with a small numerical example to build intuition.
Given a feature vector $x \in \R^{d}$ we compute a $k$ dimensional linear representation $W x$ using a matrix $W \in \R^{k \times d}$. We choose $W$ such that we best approximate a set of ground truth task vectors, $\{c_1, c_2, ..., c_T\}$, that lie in $\R^d$. The learned approximation is $\hat{c}_i= W^T \gamma_i$. Essentially, we use linear combinations of the columns of $W^T$ to approximate the task vectors. For simplicity, we assume that the columns of $W^T$ are unit norm.  We study the case where $k<T$, otherwise there are infinite solutions.

Assume we work in $d=3$ dimensions with $T=3$ total tasks, $c_1=[1,0,0]^T, c_2=[0,1,0]^T, c_3=[0,0,1]^T$. Set our learned representation dimension to be $k=1$ dimensional. When $T=2$, using only the first two tasks $c_1,c_2$, an optimal solution is $W= \frac{1}{\sqrt{2}} [1,1,0]$. The corresponding linear head is now the scalar $\gamma_1=\frac{1}{\sqrt{2}}= \gamma_2$ and the approximate vectors are $\hat{c}_1=W^T \gamma_1= [0.5, 0.5, 0 ]^T= \hat{c}_2$. Therefore the best one dimensional subspace to jointly approximate $c_1,c_2$ is the span of $W=\frac{1}{\sqrt{2}}[1,1,0]$. 
Now we introduce one more task and find the one dimensional subspace that best approximates $c_1,c_2,c_3$. That becomes $W'=\frac{1}{\sqrt{3}}[1,1,1]$ with linear heads $\gamma'_1=\frac{1}{\sqrt{3}}= \gamma'_2=\gamma'_3$.
The approximate vectors now are $\hat{c}'_1=(W')^T \gamma'_1= [1/3, 1/3, 1/3 ]^T= \hat{c}'_2=\hat{c}'_3$. Notice that $||\hat{c}'_i||^2 = 1/3$ for $3$ tasks but $||\hat{c}_i||^2 = 1/2$ for two tasks. The point is that for \textit{more tasks, the vector that jointly approximates all task vectors becomes shorter}. Equivalently, the $\ell_2$ norm of the linear task heads \textit{decreases} from $\gamma_i= \frac{1}{\sqrt{2}}$ to $\gamma'_i=\frac{1}{\sqrt{3}}$ as the tasks increased from two to three showing how multitasking creates regularization. A graphical representation of this example is given in Figure \ref{fig:3dexample}. It is important that the task vectors $c_i$ are orthogonal, increasing the effective dimensionality of the problem. The intuition is that diverse tasks increase the effective dimension, making the best approximation vector shorter. 

Our main theoretical result is that this phenomenon is quite general and makes multitasking lead to structural robustness. We connect the norm of the approximated task vectors with robustness to weight perturbations and show that for Gaussian, independent task vectors the average norm shrinks as more tasks are added. This is intuitive since high dimensional Gaussian vectors are near-orthogonal. Surprisingly, we empirically show that task vectors for numerous problems also exhibit this behavior.

\noindent \textbf{Analysis.}
We consider a neural network $f_\theta:\R^{d}\to\R^{k}$ and a collection of tasks $\{\mathcal T_1, ..., \mathcal T_T\}$. We are trying to learn $\theta, \gamma_i \in \R^{k}$ to solve the following optimization problem:
\begin{gather}
    \argmin_{\theta, \{\gamma_1, ..., \gamma_{T}}\} \sum_{i=1}^{T}\E_{(x, y) \in \mathcal T_i}\mathcal L(\gamma_i^Tf_\theta(x), y). 
    \label{eq:opt_problem_with_samples}
\end{gather}

The neural network $f_\theta$ can be as simple as a single matrix $W: \R^d \to \R^k$. For linear networks, we consider the following dataset generation process: for task $\mathcal T_i$, we sample a Gaussian $x$ and we generate its label $y$ by taking the inner-product with a task vector $c_i$, i.e. $y = c_i^Tx$ for task $\mathcal T_i$. Given infinite samples and MSE loss, the optimization problem of \eqref{eq:opt_problem_with_samples} is equivalent to the following problem.

\begin{definition}[Optimization Problem] Let $k < T < d$. We define the Factorized Best Rank-$k$ approximation of a matrix $C\in \R^{d\times T}$ as the optimization problem:
\begin{gather}
    W^*, \Gamma^* = \argmin_{W \in \R^{k\times d}, \Gamma\in \R^{k \times T} }\left|\left| W^T\Gamma - C\right|\right|_F^2.
    \label{eq:opt_problem}
\end{gather}

\label{def:opt_problem}
\end{definition}
We are interested in the case when the dimensionality of the representation $k$ is smaller than the number of tasks $T$, otherwise the best Rank-$k$ approximation of $C$ is not unique. 

The following Proposition states that in the considered setting, Problem \ref{eq:opt_problem} can be solved with SVD.
\begin{proposition}
For any matrix $C \in \R^{d\times T}$ with distinct singular values, any solution of \ref{def:opt_problem} satisfies:
\begin{gather}
    W^{*{^T}}\Gamma^* = U\Sigma_k V^T,
\end{gather}
where $U\Sigma V^T$ is the SVD of $C$ and $\Sigma_k$ is the same as $\Sigma$ except than the last $T-k$ diagonal entries that are zeroed out. 
\end{proposition}

The fact that the Singular Value decomposition computes the best rank-$k$ approximation to a matrix can be found in several textbooks e.g. \citet{golub1996matrix,blum2020foundations}.

This proposition establishes that $W^* = U^T$ and $\Gamma^* = \Sigma_kV^T$ is a valid solution of \eqref{eq:opt_problem}. Onwards, we will be calling this the SVD Solution.
\begin{definition}
We define the SVD solution of \eqref{eq:opt_problem}, to be:
\begin{gather}
        \Wsvd = U^T, \quad \Gsvd = \Sigma_kV^T.
        \label{eq:svd_solution}
\end{gather}
\end{definition}

We note that if any multitask learning algorithm is used to obtain $W^*, \Gamma^*$, one can run Gram-Schmidt to make $W^*$ orthonormal and hence 
obtain the factorization we use. It is important that $W$ stays normalized and all scaling is pushed to $\Gamma$ since to measure robustness to weight shifts, we are going to add noise to $W$ only, and higher $W$ scaling is equivalent to lower effective noise.   

We study how the performance is affected when the representation network, $f_\theta$, is corrupted.

\begin{definition}[]
For any sample $x$, the \textbf{Mean Squared Error (MSE)} for task $i$ is defined to be the expected error between the model prediction under noise and the true value $y$. Namely,
\begin{gather}
    \text{MSE}^i= \E_{\theta_c}\left[(\gamma_i^Tf_{\theta_c}(x) - y)^2 \right],
\end{gather}
where $f_{\theta_c}$ is the model that emerges after corrupting $f_{\theta}$.
\end{definition}
This measures how well the model approximates the ground truth under the presence of noise and under the constraint of a joint representation for multiple tasks.

The simplest corruption process to study is adding noise to the representation matrix, i.e. 
\begin{gather}
    W_c = W+N, \quad N_{ij} \sim \mathcal N(0, \sigma^2), \ \mathrm{i.i.d}\enspace
    \label{eq:additive_noise_model}
\end{gather}
Then, we denote the mean squared error for the task $i$ with $\msei{T}{\sigma^2}$ and the average mean squared error across the $T$ tasks with $\mse{T}{\sigma^2}$. We are now ready to introduce our results.

\begin{theorem}[Mean Squared Error for Additive Noise]
    Let $C \in \R^{d\times T}$ be a matrix with distinct singular values $\sigma_1 > \sigma_2 > ... > \sigma_T$. Let $W, \Gamma$ be the SVD solution of \eqref{eq:opt_problem}. Under the Additive Noise Model defined in \eqref{eq:additive_noise_model}, we have that: \\
    \begin{gather}
        \tikzmarknode{noisymse}{\highlight{red}{$\mse{T}{\sigma^2}$}} = \tikzmarknode{noiselessmse}{\highlight{blue}{$\mse{T}{0}$}} + \frac{\sum_{i=1}^{k}\sigma_i(C)^2}{T} \cdot \tikzmarknode{noise}{\highlight{red}{$\sigma^2$}}\enspace.
        \label{eq:th_eq}
    \end{gather}
    
    \begin{tikzpicture}[overlay,remember picture,>=stealth,nodes={align=left,inner ysep=2pt},<-]
    \path (noisymse.south) node[anchor=north east,color=red!67] (scalep){\textbf{Average MSE under noise}};
    \draw [color=red!87](noisymse.south) |- ([xshift=-0.3ex, yshift=-3ex,color=red]scalep.north west);
    
    \path (noiselessmse.north) ++ (0, 7pt) node[anchor=south west,color=blue!67] (scalep){\textbf{Average MSE without noise}};
    \draw [color=blue!87](noiselessmse.north) |- ([xshift=-0.3ex,color=blue]scalep.south east);
    
    \path (noise.south) node[anchor=north west,color=red!67] (scalep){\textbf{Noise Variance}};
    \draw [color=red!87](noise.south) |- ([xshift=-0.3ex, yshift=-3ex,color=red]scalep.north east);
    
    \end{tikzpicture}

    \label{th:additive_noise}
\end{theorem}

As shown, the noisy MSE decomposes into the sum of the noiseless MSE plus the noise variance times a function that depends on the number of tasks: 
\begin{gather}
    R(T)=\frac{\sum_{i=1}^{k}\sigma_i(C)^2}{T}.
    \label{eq:slope}
\end{gather}

It is important to emphasize that as more tasks are added, the matrix $C$ changes, but the interlacing theorem allows us to connect the singular values of smaller submatrices, as discussed in the Appendix. $R(T)$ is the robustness slope: if a model with $T$ tasks has smaller slope, it will eventually outperform a model with, say $T-1$ tasks and larger slope, for sufficiently large noise. This is true even if the noiseless performance for the $T-1$-task model is better, indicating a cross-over in MSE. 
Therefore the key is understanding when the sum of the top $k$ singular values of $C$ scales sublinearly in $T$. This is not true for tasks that are aligned, but we can show it holds for independent Gaussian task vectors. We believe it holds for more general families of diverse task vectors and our experiments verify it also holds for numerous real task vectors learned from text and vision datasets.

\paragraph{Connection with $l_2$ regularization.} For the SVD solution (see Definition \ref{eq:svd_solution}), the sum of the top-k singular values squared is the squared Frobenius norm of $\Gamma$. Indeed, we have that $||\Gamma_{\mathrm{SVD}}||_F^2 = ||\Sigma_k V^T||_F^2$. Since $\Sigma_k$ is a diagonal matrix, each row of $\Sigma_k V^T$ is a rescaling of the corresponding row of $V^T$. Rows of $V^T$ have norm $1$, hence the i-th row of $\Sigma_kV^T$ will have norm $\sigma_i$. The Frobenius norm squared is just the sum of the squared norms of the rows. Hence, we get that 
\begin{gather}
    ||\Gamma_{\mathrm{SVD}}||_F^2 = \sum_{i=1}^{k}\sigma_i(C)^2.
\end{gather}

Using this simple observation, we can get the following alternative expression of Theorem \ref{th:additive_noise}.

\begin{corollary}[]
    Let $C \in \R^{d\times T}$ be a matrix with distinct singular values. Let $W, \Gamma$ be the SVD solution of \eqref{eq:opt_problem}. Under the Additive Noise Model defined in \eqref{eq:additive_noise_model}, we have that:
    \begin{gather}
        \mse{T}{\sigma^2} = \mse{T}{0} + \frac{||\Gamma||_F^2}{T} \sigma^2\enspace.
    \end{gather}
    \label{cor:additive_noise}
\end{corollary}

Corollary \ref{cor:additive_noise} provides two important insights: i) the normalization with the number of tasks that appears in \eqref{eq:th_eq} is justified since the Frobenius norm of $\Gamma$ grows with the number of task, ii) if we can prove that the slope (defined in Equation \eqref{eq:slope}) is dropping, then we are effectively proving that multitasking gives $l_2$ regularization as we showed in the toy introductory example. This also holds for the case of Gaussian, i.i.d. task vectors, as shown in the following theorem.

\begin{theorem}
Let $C \in \R^{d\times T}$ be a random matrix with Gaussian, i.i.d. entries of variance $1/d$ and $d = \Omega(T^3)$. Let $C_t, C_{t+1}$ be the matrices formed by selecting the first $t, (t+1)$ columns of $C$. Then, there is a noise level $\sigma_{\thres}$ such that with probability $\geq 1 - \exp\left(-\Omega\left(\sqrt{d}\right)\right)$, the SVD solutions (see \eqref{eq:svd_solution}) of \eqref{eq:opt_problem} (for $C_t, C_{t+1}$ respectively), under the noise corruption model, satisfy:
\begin{gather}
    \mse{t+1}{\sigma^2} < \mse{t}{\sigma^2},\quad \forall  \, \sigma \geq \sigma_{\thres}.
\end{gather}
\label{theorem:additive_noise_random}
\end{theorem}

\begin{remark}
In words, this result shows that adding new tasks gives \textbf{provably} increased robustness to high noise corruption in the weights, when the task vectors are Gaussian.
\end{remark}
\begin{remark}
Observe that the MSE under noise drops for \textit{every single new task added}. The assumption $d = \Omega(T^3)$, can be relaxed to $d=\Omega(t^3)$, and we get increased robustness for the first $t$ added tasks. Nevertheless, for most applications $d=\Omega(T^3)$ is a realistic assumption: Even for our smallest dataset MNIST $d=728$, and we experiment with up to $10$ tasks. 
\end{remark}

\section{Experimental Evaluation}
\label{sec:experimental_eval}
We divide the experimental section in two parts.
In the first part, we add noise to the final linear representation layer of various networks and verify that our theoretical analysis agrees with experimentally observed multitasking robustness on real datasets (MNIST, CIFAR10, NewsGroup20). In the second part, we show that multitasking leads to robustness to general weight corruptions
in any layer of a complex transformer. Specifically, we show that multilingual Language Models are more robust to weight shifts (across all the layers) compared to monolingual trained under the same setting. This is the first evidence of increased Cognitive Reserve in bilingual artificial neural networks.

\paragraph{Experiments with Linear Representation Layers.} We perform experiments on three datasets (MNIST, CIFAR10, Newsgroup20) and two modalities (Vision and Language). 
The datasets normally involve one classification task each.
We create multiple binary tasks by distinguishing between pairs of labels. For example, in CIFAR10, one task might be to distinguish between dogs and cats and another between airplanes and cars. We assign a value in $[0, 1]$ to each sample for each task to transform them to regression tasks (to match our theory). For example, if task $i$ is to distinguish between dogs and cats, value $0$ corresponds to dog and value $1$ to cat.

The second issue is learning the task vectors from training data. For MNIST, we can simply learn a linear layer $C$ with columns $\{c_1, ..., c_T\}$ such that: $c_i^Tx \approx y$ for each task. For more complex datasets like CIFAR or Newsgroup20, linear networks have lower performance and hence it is less interesting to examine their robustness. Instead, we first use another network to extract representations $g_\theta(x)$ and then learn a linear layer acting on the encodings such that $c_i^Tg_{\theta}(x)\approx y$. For CIFAR we used a pre-trained Resnet50 as the encoder while for NewsGroup, a pre-trained BERT~\citep{devlin2018bert}. We would like to point out that our theory is still valid for this case -- this is equivalent to the linear layer $C$ receiving inputs from a learned representation as opposed to the features directly. 
As the number of tasks increase, we reduce the number of training examples per task. We do this to make sure that the total training dataset size stays the same as the number of tasks increase.

Figure \ref{fig:mse_linear} shows how the average MSE behaves as noise increases for different number of tasks. Note that even though all models begin from roughly the same performance in the noiseless setting, the multitask models are much more robust to the corruption of their weights consistently among all the datasets and modalities. This is aligned with our theoretical analysis which predicts that the robustness slope (defined in Equation \eqref{eq:slope}) decreases with the number of tasks. We calculate robustness slopes for learned task vectors for real datasets and plot their decay in the Appendix, where we further include all the details of how these models were trained.

\begin{figure*}[!ht]
\begin{center}
\begin{subfigure}
\centering
    \includegraphics[width=4.3cm]{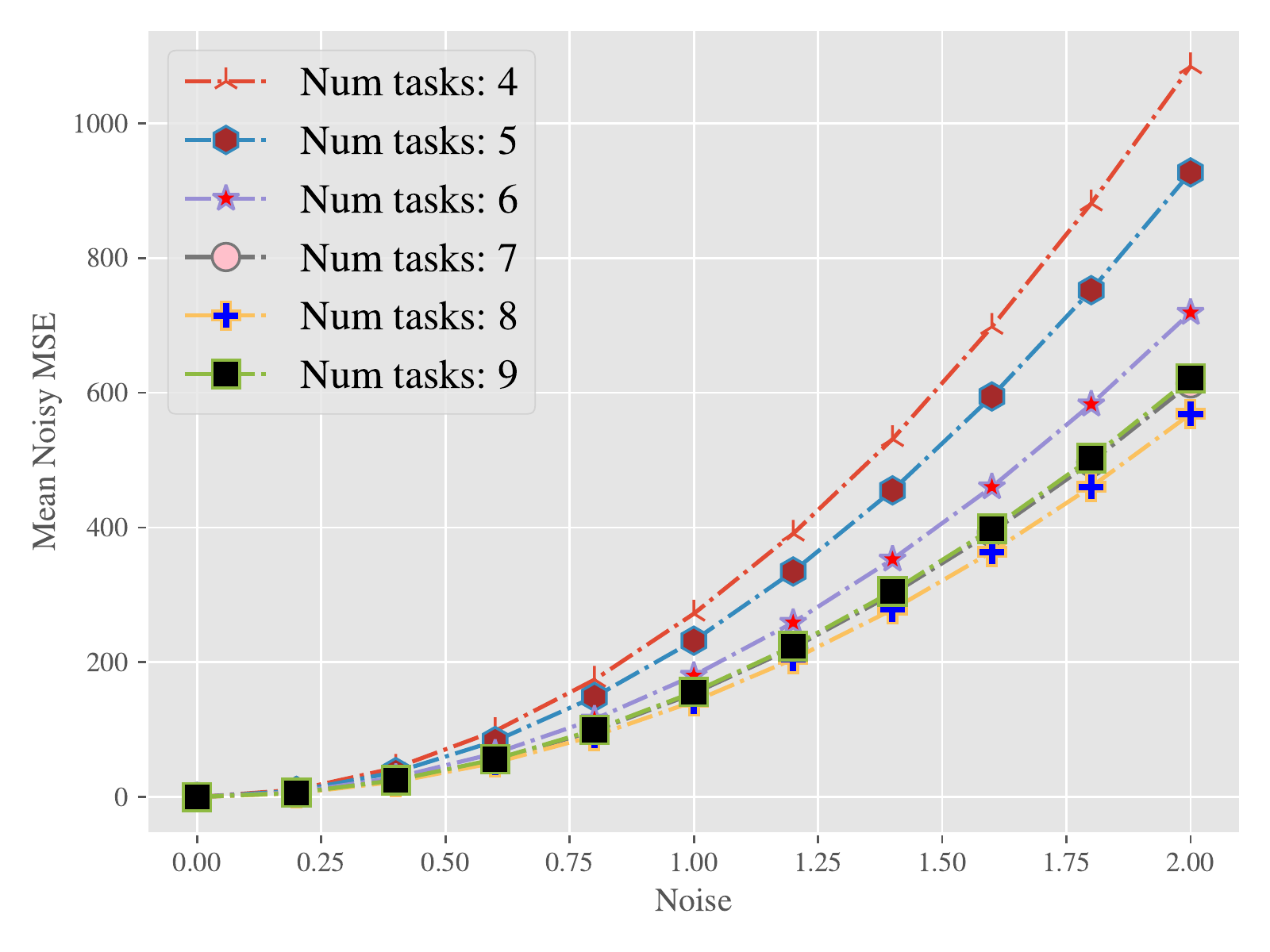}
\end{subfigure}
\begin{subfigure}
\centering
   \includegraphics[width=4.3cm]{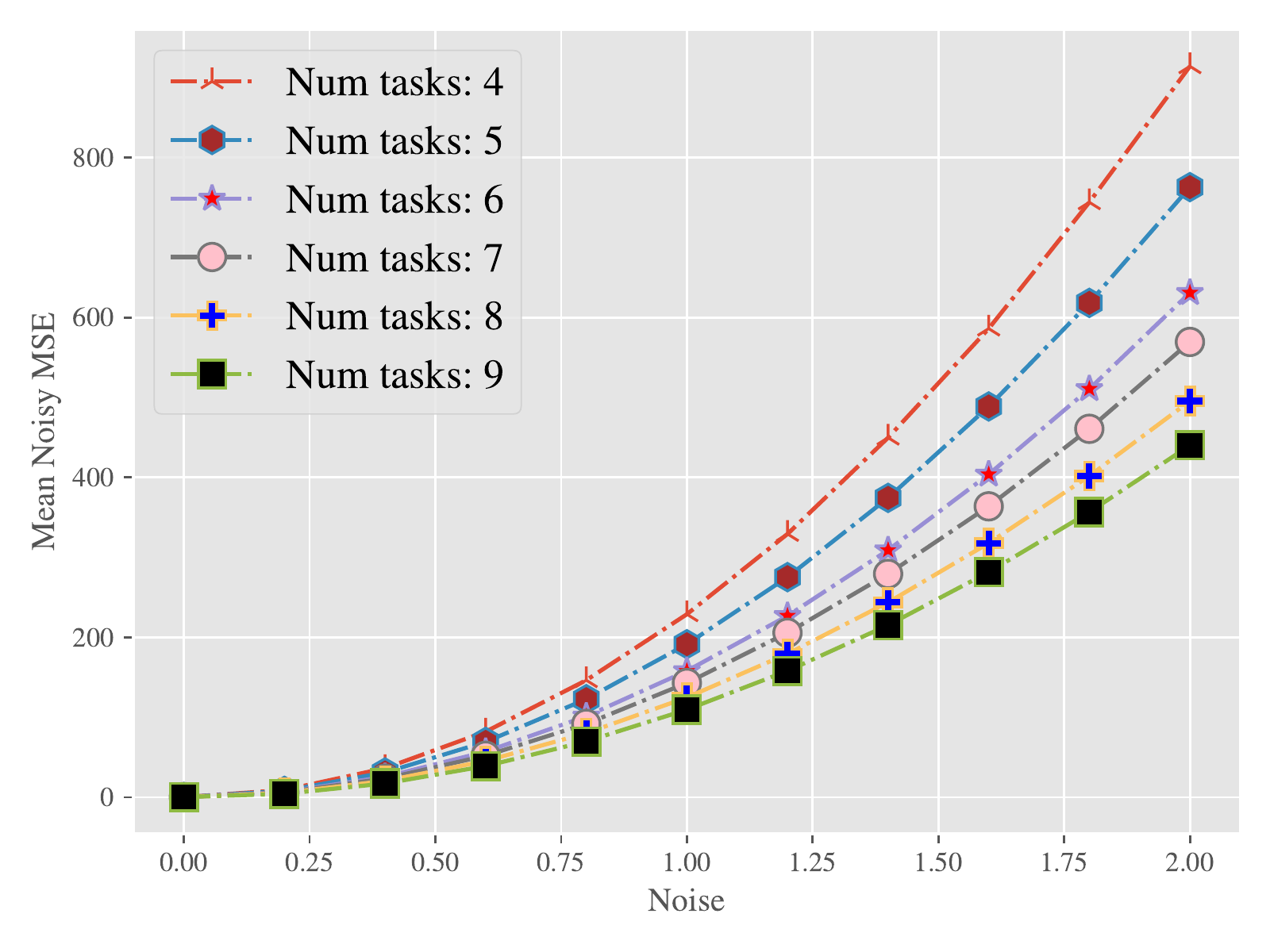}
\end{subfigure}     
\begin{subfigure}
\centering
\includegraphics[width=4.3cm]{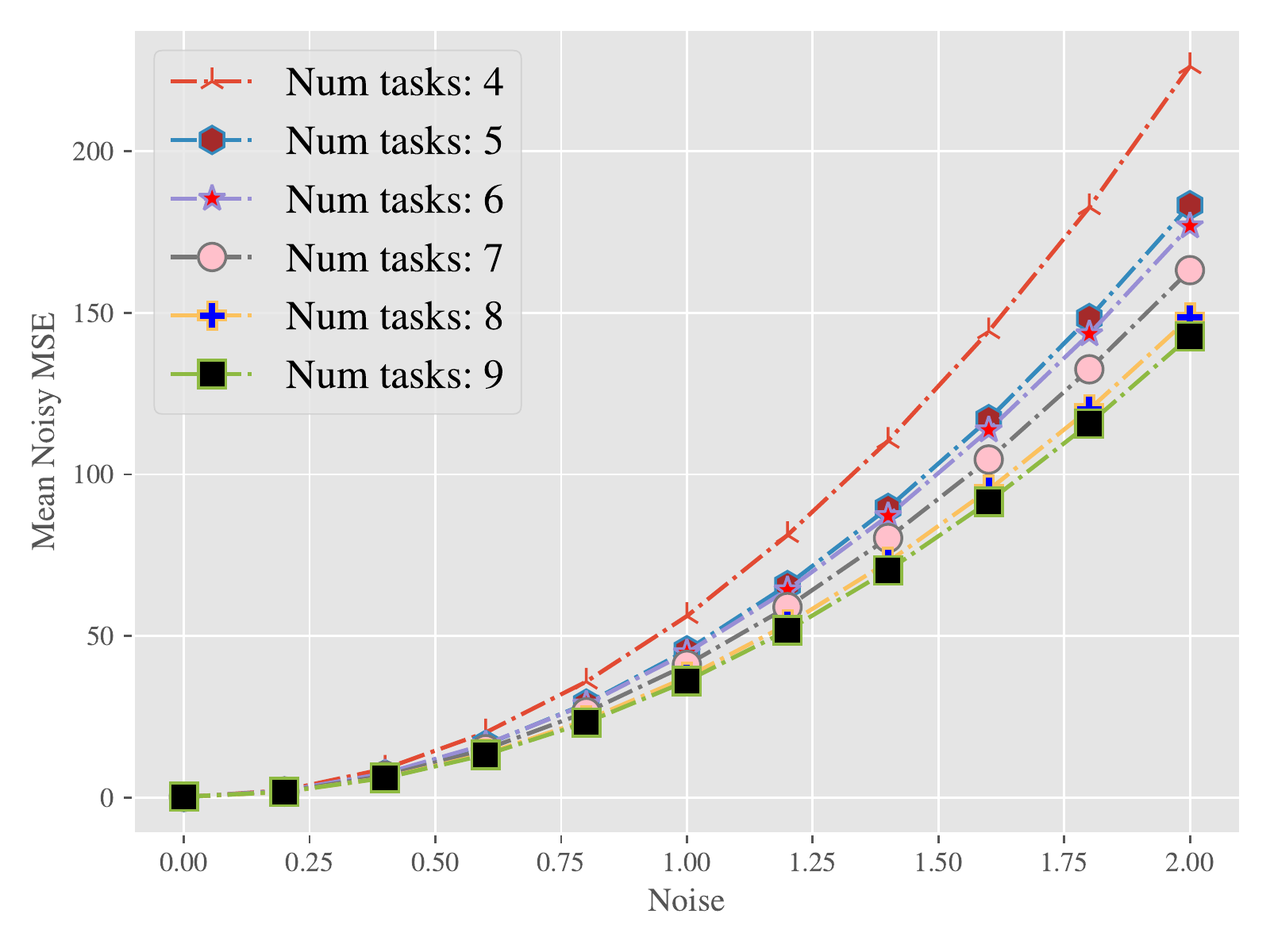}
\end{subfigure} 
\caption{\small{MSE of model (versus optimal task vector) as a function of noise added to the weights. From left to right: MNIST, CIFAR10, NewsGroup20. As shown for all these datasets, adding tasks increases the robustness of the model to noise in the weights. }}
\label{fig:mse_linear}
\end{center}
\end{figure*}

\paragraph{Experiments with Language Models.}
\label{sec:training_exp}
Our objective is to compare robustness to neural weight perturbations in monolingual and bilingual language models. We use the following perturbation models: 1) Random deletion of weight parameters:  we zero-out $p$ percent of the attention layer weights, 2) Magnitude pruning: we sort model attention weights by the magnitude and delete the smallest p percent of weights~\citep{han}, 3) Random normal noise: we add zero-mean random Gaussian noise with standard deviation $\sigma^2$ to the attention weights.

On the selection of the linguistic pair, we selected 
Greek, a highly inflected language with very different morphology, syntax and phonology compared to English. It also uses a different script since Greek characters were not Romanized. This minimizes transfer between languages, something we wanted to avoid. In the Appendix, we present additional experiments for other Romance languages.

The dataset for the bilingual model is a concatenation of articles from English and Greek Wikipedia. To avoid the computational cost of training for a new language, we start from the pre-trained GPT-2 (small)\citep{radford2019language} and we use the Language Model Recycling Technique, introduced in \citep{gpt_recycling}. GPT-2 small is a transformer-based architecture for causal language modeling, with $12$ attention blocks and $124$M parameters. The tokenizer uses Byte Pair Encoding and has a vocabulary of $50,257$ tokens. For the bilingual model, we generate a new tokenizer, vocabulary and embedding layer without changing the architecture. We keep the vocabulary size the same, as changing the vocabulary size can affect the scale of the perplexity score for these models. Note that Wikipedia documents were not in the original training of GPT-2, but our monolingual baseline was subsequently finetuned on English Wikipedia. Details on all our training hyperparameters are included in the Appendix.

We measure the quality of generated text using perplexity. Our bilingual model achieves $89$ perplexity on a randomly picked subset of the OSCAR~\citep{OrtizSuarezSagotRomary2019} dataset and $76$ perplexity on the English IMDB dataset \citep{maas-etal-2011-learning}. Monolingual GPT-2 model achieves $36$ perplexity on the IMDB dataset. In the Appendix we include generated text for both the models. Although the perplexity of the bilingual model does not match the pre-trained GPT-2, the generated text is of reasonable quality text in both languages.

\paragraph{Text Generation.}
Our first experiment is to compare the performance of both models under various parameter perturbations. First, we try deleting a random portion $p$ ($p$ from $0\%$ to $40\%$) of attention layers' weight to observe and compare the trend of decay in text generation quality between the two models. We evaluate both models on the IMDB dataset. As the graph in Figure \ref{fig:gpt2_random_deletion} shows, the monolingual model starts with text predictions closer to the source text, resulting in lower perplexity without noise. However, as we delete a more significant portion of weights, the bilingual model matches the performance of the monolingual one and eventually outperforms that.

\begin{figure*}[!ht]
\begin{center}
\begin{subfigure}
    \centering
    \includegraphics[width=6.9cm]{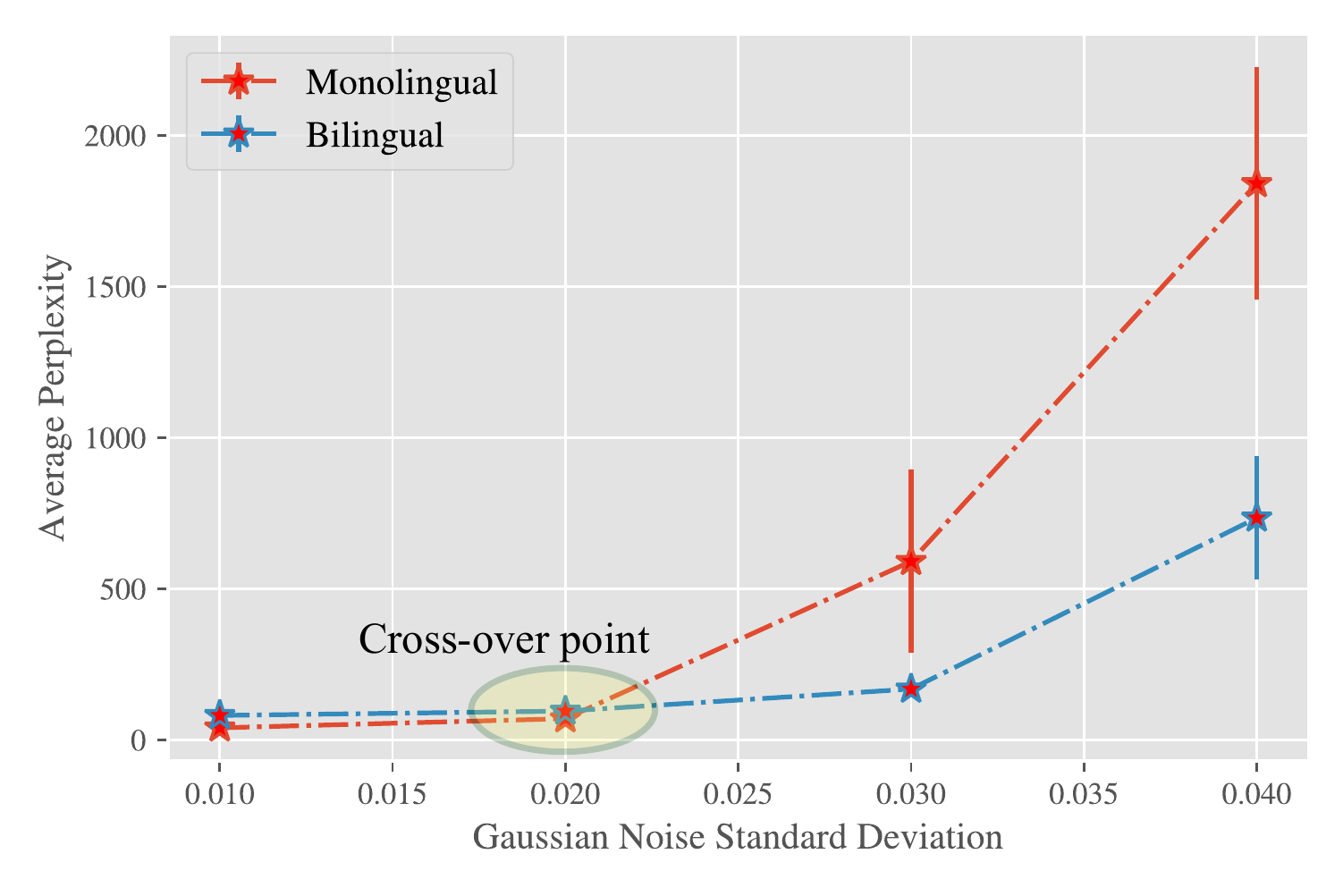}
\end{subfigure}
\begin{subfigure}
    \centering
    \includegraphics[width=6.9cm]{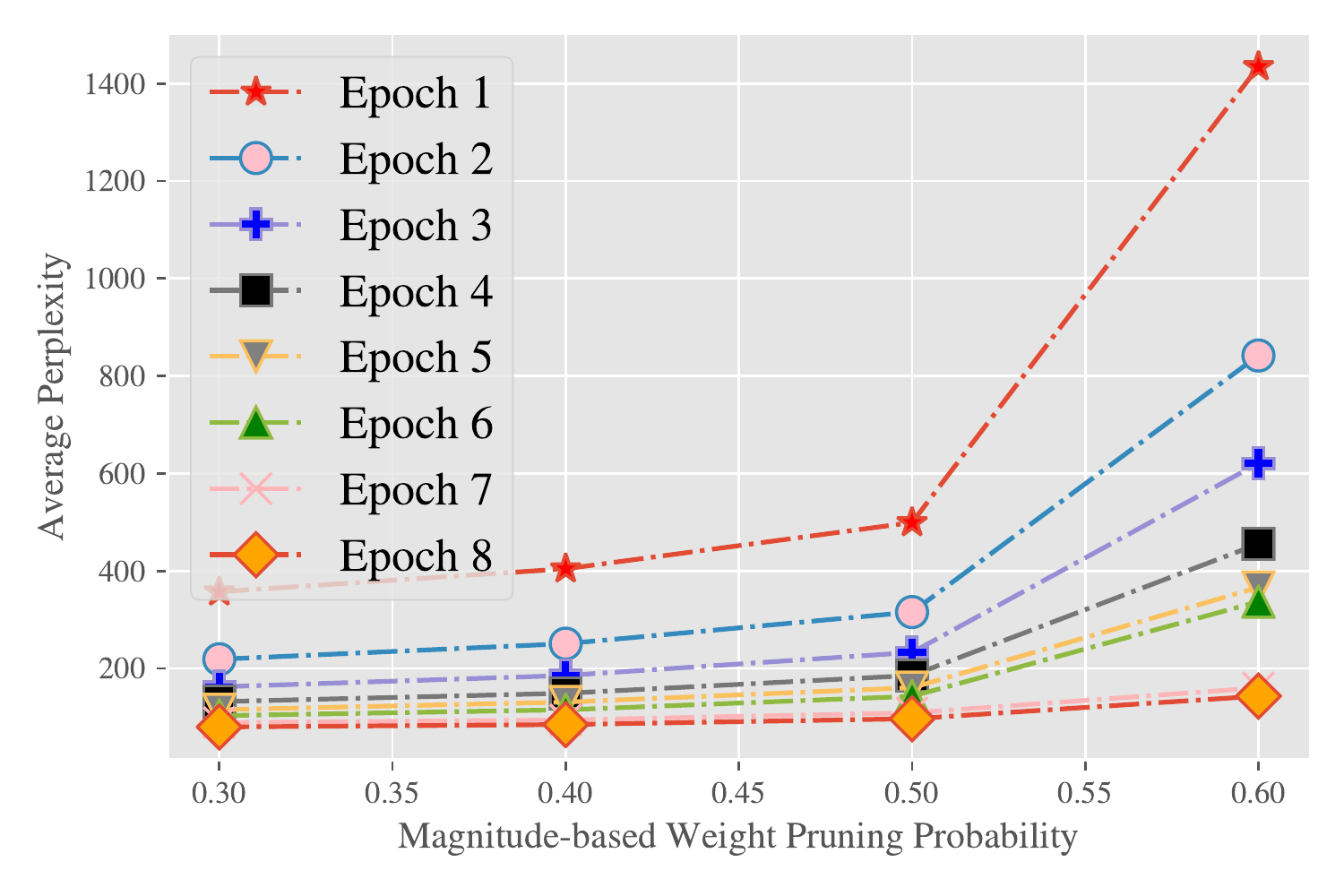}
\end{subfigure}
\end{center}
    \caption{\small Robustness to magnitude-based weight pruning and additive Gaussian noise. When plotting perplexity under additive Gaussian noise, x-axis indicates the standard deviation of noise added to weights. Y-axis indicates the average perplexity over $20$ runs with $95$\% confidence intervals. The second plot shows perplexity as we delete more weights based on magnitude, for the bilingual model at each epoch. X-axis indicates the probability of deleting sorted attention weight parameters. After only one epoch, the model shows higher sensitivity to weight perturbations. However, after eight epochs of training, it becomes more robust.}
\label{fig:gpt2_magnitude_deletion}
\end{figure*}

Next, we try magnitude-based pruning of a portion of weights, $p$, to observe and compare the trend of decay in text generation quality between the two models. We sort the attention layer weights by the magnitude and set $p$ percent of weights with the lowest magnitude to zero. Again, we use the IMDB dataset to evaluate models. The graphs in Figure \ref{fig:gpt2_magnitude_deletion} show that as the training process continues, the model achieves a lower perplexity. Moreover, pruning additional weights has a less substantial impact on the model's performance. This graph shows that training the pre-trained GPT-2 model for a few epochs on a bilingual dataset significantly improves robustness to weight perturbations.

In another experiment, we observe how the maximum singular value of the weight matrices changes throughout training process. We track the maximum singular value of attention layer weights. We use a pretrained GPT-2 model baseline, and train this model for 16k iterations on English text data from Wikipedia. Resuming from this checkpoint, we train two new models: 1) We continue training model 1 on task 1 (English Wikipedia dataset) for 16k more iterations. 2) We train a second model on a different English dataset, the LAMBADA dataset \citep{paperno2016lambada},  for 16k more iterations. Figure \ref{fig:sv_drop} indicates the results of this experiment by plotting maximum singular values of the first attention layer. As the Figure shows, training model on a new dataset (task 2) results in a faster decay of the maximum singular value.

\begin{figure}[ht]
    \centering
    \includegraphics[width=7.5cm]{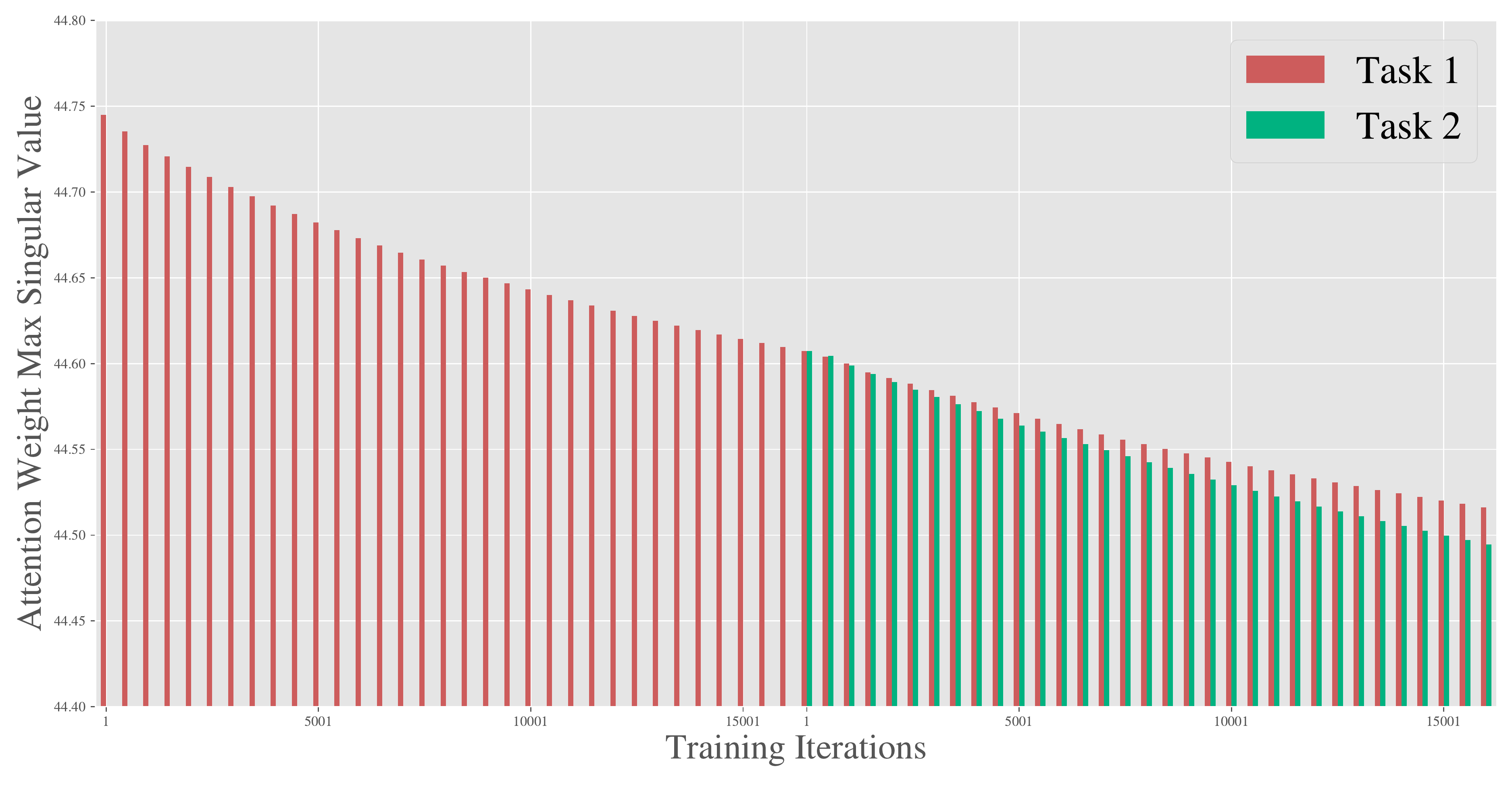}
    \vspace{-0.3cm}
    \caption{\small We show the effect of monolingual and bilingual training on the maximum singular value of attention weights. The red line shows the maximum singular value for a monolingual model trained on English Wikipedia for 32k iterations. The green line shows the maximum singular value if in the $16K$ iteration we switch to bilingual training. As shown, bilingual training leads to faster decay in the maximum singular value.}

    \label{fig:sv_drop}
\end{figure}

\paragraph{Text Classification.}
We conduct another set of experiments to observe the robustness of fine-tuned monolingual and bilingual GPT-2 models for text classification. In this section, we fine-tune both the monolingual and the bilingual GPT-2 models (previously trained) for downstream classification tasks using the GLUE benchmark
\citep{GLUE, sst, stsb, qqp, wnli, rte1, rte2, rte3, rte4, mrpc, cola, squad}
to compare the robustness of models to weight perturbations. The two perturbation methods tested in this section are random weight deletion and random Gaussian noise added to attention weights. For each task, we fine-tune both models for ten epochs. 
When applying random pruning, the accuracy of each model is evaluated after deleting $p$ percent of model weights, $p$ ranging from $0\%$ to $45\%$. When perturbing model weights by adding noise, we try various Gaussian noise distributions with standard deviations ranging from 0 to 0.09. Experiment results can be found in the Appendix section.

\begin{table}
    \centering
\begin{tabular}{c|c|c}
Task & \multicolumn{2}{c}{Fine-tuned using} \\
 & Monolingual ckpt & Bilingual ckpt\\ 
SST2 & 70567.875 & 60663.121 \\
QQP & 70608.195 & 60649.586 \\
MRPC & 70498.953 & 60590.769 \\
RTE & 70508.968 & 60590.765 \\
CoLA & 70519.781 & 60600.933 \\
\end{tabular}
     \caption{\small We compute the sum of the squares of the weights of an attention layer for monolingual and bilingual models. The latter have smaller magnitudes, indicating that multitasking induces weight regularization. }
     \label{table:sum_sqr_weights}
\end{table}

\paragraph{Random Pruning.}
 We compare the classification accuracy between the fine-tuned model from the monolingual pre-trained network and the fine-tuned model using the bilingual network. Each element in attention parameters is pruned with probability $p$, where $p$ ranges from $.0$ to $.45$. We evaluate the classification accuracy for the following GLUE tasks: CoLA, QQP, SST2, MRPC, QNLI, and RTE. 

We expect the accuracy of both models to decay as we prune a more considerable number of parameters. The monolingual model shows a faster decay in almost all tasks. For some tasks such as SST2, QQP, and MRPC, we observe that the bilingual model starts with lower accuracy, and its performance exceeds the monolingual model as we prune $\approx{5\%}$ to $\approx{25\%}$ of parameters. A detailed set of results in Table \ref{table:random_pruning_table} show models' average prediction accuracy on the GLUE benchmark. 

\definecolor{forestgreen}{rgb}{0.0, 0.7, 0.5}

\begin{table*}[t]
    \centering
\small
\begin{tabular}{c|cc|cc|cc|cc|cc}
Pruning Probability & \multicolumn{2}{c}{QQP}  & \multicolumn{2}{c}{SST2} & \multicolumn{2}{c}{COLA} & \multicolumn{2}{c}{MRPC}  & \multicolumn{2}{c}{RTE} \\
 & m. & b. & m. & b. & m. & b. & m. & b. & m. & b. \\
0.00 & 0.876 & 0.843 & 0.908 & 0.862 & 0.437 & 0.218 & 0.828 & 0.774 & 0.646 & 0.595 \\
0.05 & 0.873 & 0.842 & 0.909 & 0.866 & 0.425 & 0.203 & 0.804 & 0.769 & 0.640 & 0.589 \\
0.10 & 0.867 & 0.833 & 0.899 & 0.868 & 0.403 & 0.204 & \color{forestgreen}0.730 & \color{forestgreen}\textbf{0.744} & 0.603 & 0.575 \\
0.15 & 0.848 & 0.819 & 0.871 & 0.866 & 0.366 & 0.185 & 0.619 & \textbf{0.730} & 0.600 & 0.562 \\
0.20 & 0.804 & 0.786 & \color{forestgreen}0.836 & \color{forestgreen}\textbf{0.859} & 0.326 & 0.179 & 0.416 & \textbf{0.663} & 0.561 & 0.553 \\
0.25 & \color{forestgreen}0.711 & \color{forestgreen}\textbf{0.732} & 0.806 & \textbf{0.847} & 0.267 & 0.159 & 0.377 & \textbf{0.653} & 0.543 & 0.546 \\
0.30 & 0.656 & \textbf{0.678} & 0.760 & \textbf{0.828} & 0.216 & 0.137 & 0.320 & \textbf{0.504} & 0.537 & 0.536 \\
0.35 & 0.638 & \textbf{0.674} & 0.714 & \textbf{0.815} & 0.153 & 0.092 & 0.317 & \textbf{0.420} & 0.522 & 0.494 \\
0.40 & 0.632 & \textbf{0.655} & 0.683 & \textbf{0.793} & 0.097 & 0.058 & 0.316 & \textbf{0.328} & 0.521 & 0.488 \\
0.45 & 0.632 & \textbf{0.636} & 0.651 & \textbf{0.773} & 0.060 & 0.042 & 0.316 & \textbf{0.328} & 0.525 & 0.485 \\
\end{tabular}
     \caption{\small Performance under a range of random pruning probabilities for various GLUE tasks. Columns labeled with ``m" determine classification accuracy of  monolingual models and columns labeled as ``b" determine accuracy of bilingual. CoLA is evaluated using Matthew's Correlation and other tasks are evaluated by accuracy.}
     \label{table:random_pruning_table}
\end{table*}

\paragraph{Random Noise.}
We also experiment with adding Gaussian noise to the weights.
We vary the noise standard deviation from .0 to 0.09. We evaluate the classification accuracy for the same tasks.
When no noise is added to model parameters, the monolingual model performs slightly better for tasks like QQP and SST2. As we increase the noise, the accuracies of both models drop with almost identical rates. However, both graphs illustrate a cross-over point after which the bilingual model outperforms the monolingual.
The bilingual model achieves significantly higher accuracy in the MRPC task when the standard deviation is greater than $\approx{0.03}$. For CoLA and RTE, the monolingual
model maintains maintains higher performance regardless of
the noise level. A detailed set of results in the Appendix section shows models' average prediction accuracy on the GLUE benchmark.

\begin{figure*}[!htp]
\begin{center}
\begin{subfigure}
\centering
    \includegraphics[width=0.44\linewidth]{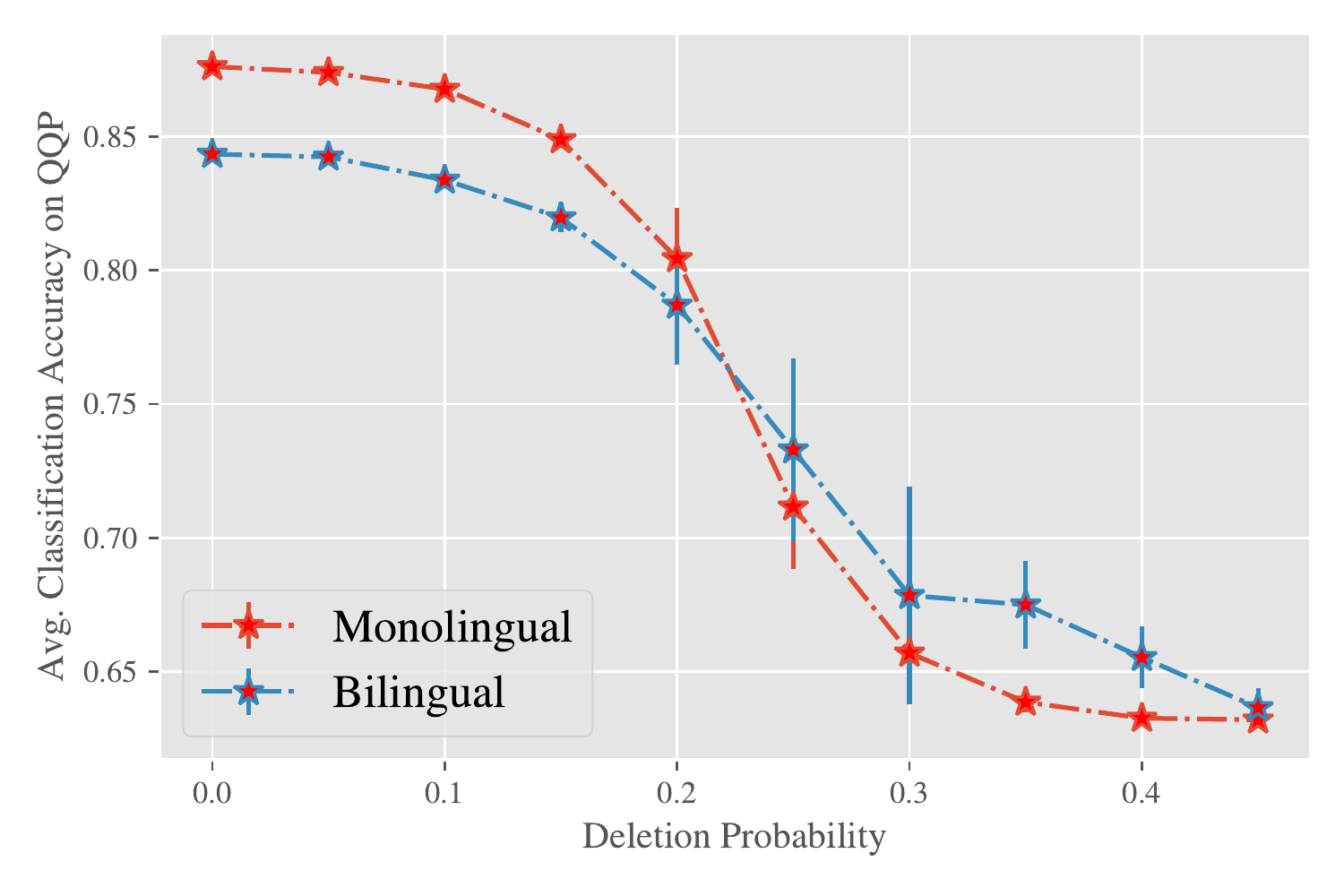}
\end{subfigure}
\begin{subfigure}
\centering
    \includegraphics[width=0.44\linewidth]{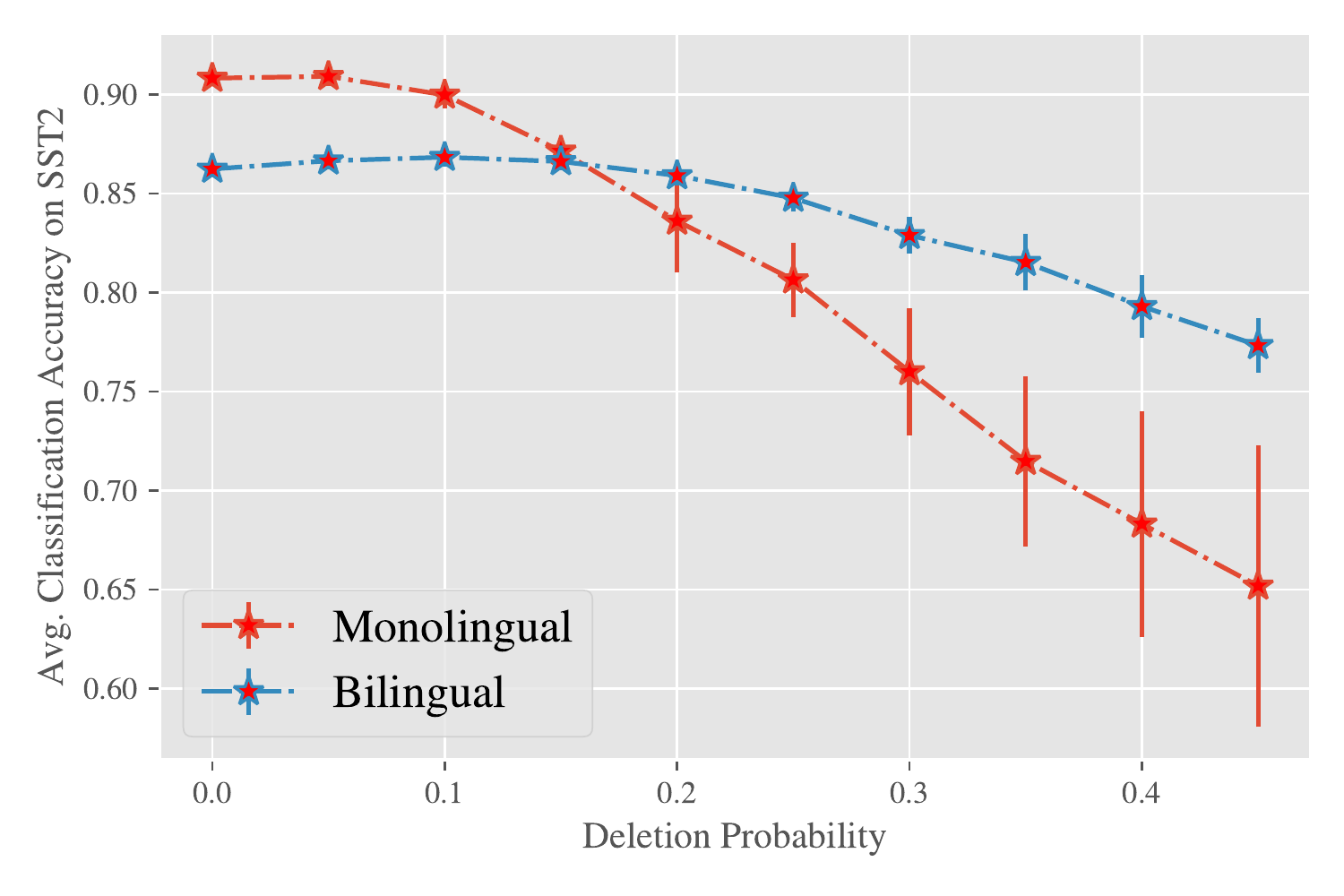}
\end{subfigure}
\begin{subfigure}
\centering
    \includegraphics[width=0.44\linewidth]{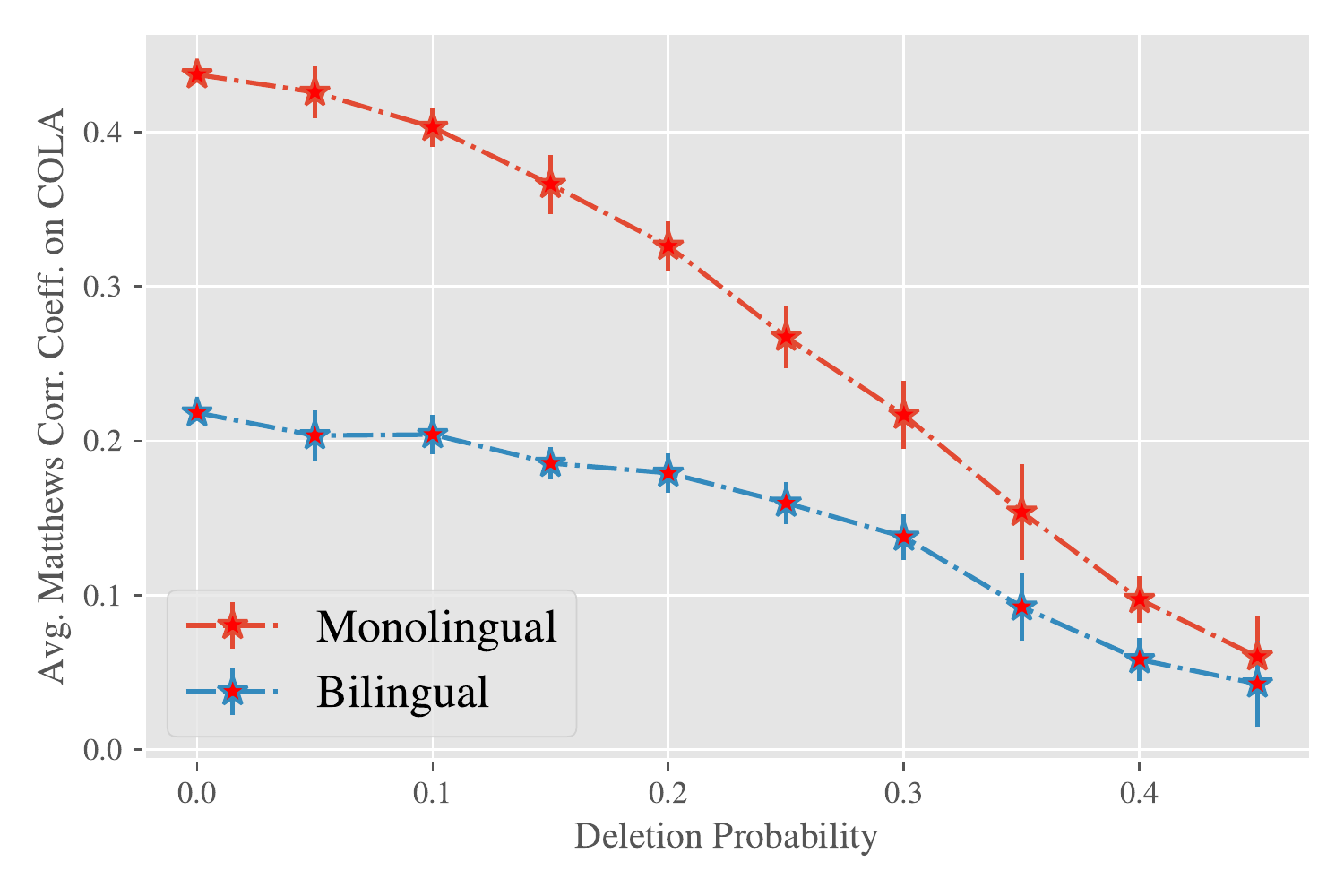}
\end{subfigure}
\begin{subfigure}
\centering
    \includegraphics[width=0.44\linewidth]{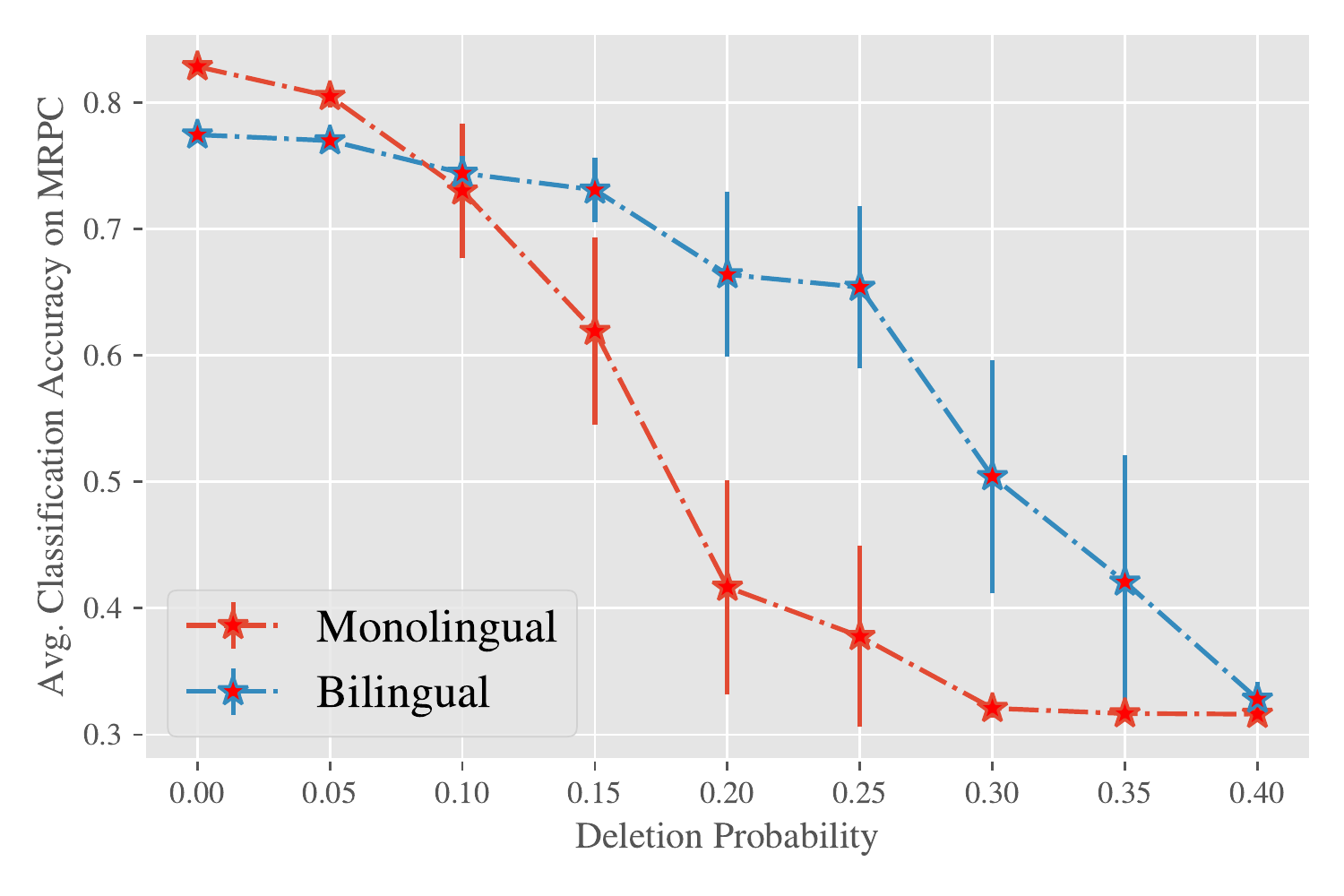}
\end{subfigure}
\vspace{-0.5cm}
\end{center}
\caption{\small Performance comparison in GLUE tasks: QQP, SST2, CoLA, and MRPC under random erasures. QQP: Monolingual drops lower than the bilingual model after $\approx25\%$ of the parameters are deleted. SST2: Monolingual drops with a faster rate, falling behind the bilingual after deleting $\approx15\%$ of the parameters. 
CoLA: Both models reach $\approx 0$ MCC (random prediction) with $\approx 45\%$ of parameters pruned. MRPC: The accuracy of the monolingual degrades at a faster rate as pruning probability increases higher than $\approx 10\%$.
}
\end{figure*}

\section{Related Work}

\noindent \textbf{Cognitive Reserve and Bilingualism.} 
Our work is inspired by Cognitive Science and evidence of Cognitive Reserve in bilinguals. 
One implication of our theory is that multitasking leads to smaller weights on average. This could be related to studies performed in healthy older adults that indicate that despite overall less gray matter volume and poorer white matter integrity (i.e., poorer structural brain connectivity), older healthy bilinguals perform equally well or outperform monolinguals in several cognitive tasks \citep{anderson2021bilingualism,gold2013lifelong}. 

We would like to emphasize that our research is solely on \textit{artificial networks} which have huge differences to biological neurons. No definite extrapolations should be made to Cognitive Neuroscience without further work.
Nonetheless, we show that there is a simple mathematical abstraction that seems to align with the significantly more complex phenomena observed in bilingual cognitive reserve.

\noindent \textbf{Multitask Learning.} The most closely related work is by \citet{multitask_adv} which shows that multitask learning increases \textit{adversarial} robustness. The intuition behind their proof is that, with task diversity, the gradient of the loss with respect to the wrong label is small as orthogonal tasks make gradients that cancel out. \citet{wu2020adversarial} establishes a connection between robustness to weight perturbations and adversarial attacks. Our work is related but different since it directly establishes a connection between structural robustness and multitasking and shows a cross-over in performance across various domains and tasks. Our theoretical analysis is also completely different compared to prior works. More information on multitask learning can be found in \citet{multitask_adv} and \citet{ghamizi2021adversarial}.

Many studies on network compression and the Lottery Ticket Hypothesis are related to our Magnitude Pruning experiments.
\citet{NIPS1989_6c9882bb, han} find that selectively pruned networks can be trained from randomly initialized weights to match the performance of the original network. \citet{frankle2018the} introduces the hypothesis that randomly initialized neural networks contain a very sparse sub-network that, if initialized correctly, can achieve the accuracy of the original model. \citet{chen2021long} studies this in continual learning and examines various pruning methods.

\section{Conclusions}
\label{sec:conclusions}
We demonstrated a connection between multitask learning and robustness to structural failures for artificial neural networks. For linear representation learning we obtained a characterization of robustness through the spectrum of the task matrix. We showed that robustness comes from diverse tasks which imply a bounded spectral norm for $C$. One limitation of our theoretical work is that we did not analyze learning algorithms but directly used the SVD solution. It would be interesting to see if gradient descent introduces further regularization
or other effects, 
especially in the non-linear case. 

Experimentally, we observed increased robustness for both linguistic and non-linguistic tasks. More complex settings like multi-lingual models, cross-language transfer and their interactions remain to be explored. 
Finally, it remains open if bilingualism and cognitive reserve in humans can indeed be connected to our framework. It would be fascinating if neuroimaging techniques can measure any form of anatomical or functional regularization that bilingualism could be creating in humans. 

\section{Acknowledgments} 
This research has been supported by NSF Grants CCF 1763702,
AF 1901292, CNS 2148141, Tripods CCF 1934932, IFML CCF 2019844, the Texas Advanced Computing Center (TACC) and research gifts by Western Digital, WNCG IAP, UT Austin Machine Learning Lab (MLL), Cisco and the Archie Straiton Endowed Faculty Fellowship.

\bibliography{citations.bib}

\appendix

\newpage
\appendix
\onecolumn

\section{Proofs}
\newtheoremstyle{named}{}{}{\itshape}{}{\bfseries}{.}{.5em}{\thmnote{Theorem }#3}
\theoremstyle{named}
\newtheorem*{namedtheorem}{theorem}
\begin{namedtheorem}[\ref{th:additive_noise}]
    Let $C \in \R^{d\times T}$ be a matrix with distinct singular values $\sigma_1 > \sigma_2 > ... > \sigma_T$. Let $W, \Gamma$ be the SVD solution of \eqref{eq:opt_problem}. Under the Additive noise model defined in \ref{eq:additive_noise_model},
    \begin{gather}
        \mse{T}{\sigma^2} = \mse{T}{0} + \frac{\sum_{i=1}^{k}\sigma_i^2(C)}{T} \sigma^2\enspace.
    \end{gather}
\end{namedtheorem}

\begin{proof}
\begin{gather}
   \msei{T}{\sigma^2} = \E\left[ \left(c_i^Tx - \gamma_i^TW_cx\right)^2\right] = x^Tc_ic_i^Tx - 2c_i^Tx\gamma_i^T\E[W_c]x + x^T\E[W_c^T\gamma_i\gamma_i^TW_c]x \\ 
   = x^Tc_ic_i^Tx - 2c_i^Tx\gamma_i^TWx + x^T\E[W_c^T\gamma_i\gamma_i^TW_c]x \\
   = x^Tc_ic_i^Tx - 2c_i^Tx\gamma_i^TWx + x^T\E[(W^T + N^T)\gamma_i\gamma_i^T(W + N)]x \\ 
  = x^Tc_ic_i^Tx - 2c_i^Tx\gamma_i^TWx + x^TW^T\gamma_i\gamma_i^TWx + x^T\E[N^T\gamma_i\gamma_i^TN]x\enspace.
\end{gather}
Now, observe that:
\begin{gather}
    \sum_{i=1}^{T}\msei{T}{\sigma^2} = \left(\sum_{i=1}^{T} x^Tc_ic_i^Tx - 2c_i^Tx\gamma_i^TWx + x^TW^T\gamma_i\gamma_i^TWx\right) + x^T\E[N^T \gamma_i\gamma_i^T N]x \\
    \mse{T}{\sigma^2} = \mse{T}{0} + \frac{x^T\E\left[N^T\left(\sum_{i=1}^{T}\gamma_i\gamma_i^T\right)N\right]x}{T}\enspace.
\end{gather}

Observe that
\begin{gather}
    \E\left[N^T \left(\sum_{i=1}^{T} \gamma_i\gamma_i^T\right) N\right] = \sigma^2 \tr\left(\sum_{i=1}^{T}\gamma_i\gamma_i^T \right)I_d\enspace.
\end{gather}

For any unit-norm $x$:
\begin{gather}
    x^T\E\left[N^T \left(\sum_{i=1}^{T} \gamma_i\gamma_i^T\right) N\right]x = \sigma^2\tr\left(\sum_{i=1}^{T}\gamma_i\gamma_i^T \right)\enspace.
\end{gather}

Now for the SVD solution, we know that $\Gamma = \Sigma_k V^T$ and $\{\gamma_i\}$ are the columns of $\Gamma$. Hence, 
\begin{gather}
    \sum_{i=1}^{T}\gamma_i\gamma_i^T = \Sigma_k^2 V^TV = \Sigma_k^2.
\end{gather}

Then,

\begin{gather}
    \mse{T}{\sigma^2} = \mse{T}{0} + \sigma^2\frac{\sum_{i=1}^{k}\sigma_i(C)^2}{T}\enspace.
\end{gather}
\end{proof}

We are going to use the following result due to \citet{ledoux2001concentration}.
\begin{lemma}[\citet{ledoux2001concentration}]
Let $C_T: \R^{d\times T}$ be a random matrix whose entries are i.i.d. Gaussian with variance $1/d$. Let $C_{K}$ be the random matrix that submatrix of $C$ that consists of the first $K$ columns of $C_T$. Then,
\begin{gather}
    \Pr\left[\sigma_{\max}(C_K) \geq 1 + \sqrt{K/d} + o(1) + \alpha\right] \leq \exp(-d\alpha^2/2)
\end{gather}
and
\begin{gather}
    \Pr\left[\sigma_{\min}(C_K) \geq 1 - \sqrt{K/d} + o(1) - \alpha\right] \leq \exp(-d\alpha^2/2)\enspace,
\end{gather}
where $o(1)$ is a small-term that tends to $0$ as $d\to \infty$.
\label{lemma:gaussian_spectral_concentration}
\end{lemma}

\begin{namedtheorem}[\ref{theorem:additive_noise_random}]
Let $C \in \R^{d\times T}$ be a random matrix with Gaussian, i.i.d. entries of variance $1/d$ and $d = \Omega(T^3)$. Let $C_t, C_{t+1}$ be the matrices that are formed by selecting the first $t, (t+1)$ columns of $C$ respectively. Then, there is a noise level $\sigma_{\thres}$ such that with probability $\geq 1 - \exp\left(-\Omega\left(\sqrt{d}\right)\right)$, the SVD solutions (see \eqref{eq:svd_solution}) of \eqref{eq:opt_problem} (for $C_t, C_{t+1}$ respectively), under the noise corruption model, satisfy:
\begin{gather}
    \mse{t+1}{\sigma^2} < \mse{t}{\sigma^2}\enspace,
\end{gather}
$\forall \sigma \geq \sigma_{\thres}$.
\end{namedtheorem}

\begin{proof}
From Theorem \ref{th:additive_noise}, we have that:
\begin{gather}
            \mse{t+1}{\sigma^2} = \mse{t+1}{0} + \frac{\sum_{i=1}^{k}\sigma_i^2(C_{t+1})}{t} \sigma^2, \qquad \mse{t}{\sigma^2} = \mse{t}{0} + \frac{\sum_{i=1}^{k}\sigma_i^2(C_t)}{t}\sigma^2\enspace.
\end{gather}
To prove the desired thing, we just need to show that $\mse{t+1}{\sigma^2}$ has a smaller co-efficient for the term $\sigma^2$, because for large enough $\sigma$, eventually this term will dominate the sum. Hence, we need to show that:
\begin{gather}
    \frac{\sum_{i=1}^{k} \sigma_i^2(C_t)}{t} \geq \frac{\sum_{i=1}^{k} \sigma_i^2(C_{t+1})}{t+1}\enspace.
\end{gather}

Since $C_t$ is a submatrix of $C_T$, from the Eigenvalue Interlacing Theorem we know that $\sum_{i=1}^{k}\sigma_i^2(C_t) \leq \sum_{i=1}^{k}\sigma_i^2(C_{t+1})$. However, the difference of the two sums is upper-bounded. Using Lemma \ref{lemma:interlacing}, we get that: 
\begin{gather}
    \sum_{i=1}^{k}\sigma_i^2(C_{t+1}) = \sigma_1^2(C_{t+1}) + \sum_{i=2}^{k} \sigma_{i}^2(C_{t+1}) \\
    \leq \sigma_1^2(C_{t+1}) + \sum_{i=1}^{k-1} \sigma_{i}^2(C_{t}) \\
    = \sigma_1^2(C_{t+1}) - \sigma_{k}^2(C_t) + \sum_{i=1}^{k}\sigma_i^2(C_t).
\end{gather}

It suffices to show that:
\begin{gather}
\frac{\sum_{i=1}^{k} \sigma_i^2(C_t)}{t} \geq \frac{\sigma_1^2(C_{t+1}) - \sigma_{k}^2(C_t) + \sum_{i=1}^{k}\sigma_i^2(C_t)}{t+1} \iff \\
\sum_{i=1}^{k}\sigma_i^2(C_t) \geq t\left( \sigma_1^2(C_{t+1}) - \sigma_k^2(C_t)\right).
\end{gather}
Trivially, $\sum_{i=1}^k \sigma_i^2(C_t) \geq k \sigma_{k}^2(C_t)$. Hence, it is enough to show that:
\begin{gather}
 \sigma_k^2(C_t) \geq \frac{t}{k} \left( \sigma_1^2(C_{t+1}) - \sigma_k^2(C_t)\right).
 \label{eq:sing_diff}
\end{gather}

We will now bound the difference of the first and $k$-th singular values.

From Lemma \ref{lemma:gaussian_spectral_concentration}, we have that:
\begin{gather}
    \Pr\left[ \sigma_1(C_{t+1}) \geq 1 + o(1) + \sqrt{\frac{t+1}{d}} + \alpha\right] \leq \exp\left( -d\alpha^2/2\right)
\end{gather}
and
\begin{gather}
    \Pr\left[ \sigma_k(C_{t}) \leq 1 + o(1) - \sqrt{\frac{t}{d}} - \alpha \right] \leq \exp(-d\alpha^2/2).
\end{gather}

By union bound, with probability $\geq 1 - 2\exp(-d\alpha^2/2)$, we have that:
\begin{gather}
    \sigma_1^2(C_{t+1}) - \sigma_k^2(C_t) \leq \left( 1 + o(1) + \sqrt{\frac{t+1}{d}} + \alpha\right)^2 - \left( 1 + o(1) - \sqrt{\frac{t}{d}} - \alpha \right)^2 \\
    = 2 (1 + o(1)) \left( \sqrt{\frac{t+1}{d}} + \alpha\right) \left( \sqrt{\frac{t}{d}} + \alpha\right) + \left( \sqrt{\frac{t+1}{d}} + \alpha\right)^2 - \left( \sqrt{\frac{t}{d}} + \alpha\right)^2 \\
    \leq 5\left( \sqrt{\frac{t+1}{d}} + \alpha \right)^2.
\end{gather}

We choose $\alpha = \left(\frac{t+1}{d}\right)^{1/4}$.
Since, $t < T < d$, we have that with probability $\geq 1 - \exp\left(-\Omega\left(\sqrt{d}\right)\right)$,
\begin{gather}
    \sigma_1^2(C_{t+1}) - \sigma_k^2(C_t) \leq 20 \sqrt{\frac{t+1}{d}}, \quad \sigma_k(C_t) \leq 1 + o(1) - 2\sqrt{\frac{t+1}{d}}.
\end{gather}

Going back to Eq. \ref{eq:sing_diff}, it suffices to show that:
\begin{gather}
    1 + o(1) - 2\sqrt{\frac{t+1}{d}} \geq \frac{20t}{k} \sqrt{\frac{t+1}{d}} \iff 
    1 + o(1) \geq \sqrt{\frac{t+1}{d}}\left( \frac{20t}{k} + 2\right).
\end{gather}
Since $t < T$, this is true for $d = \Omega(T^3)$.

\end{proof}

\begin{lemma}
Let $C$ be a matrix $\in \R^{d \times T}$ and $c_{T+1} \in \R^{d}$. Let also $\Cnew = \begin{bmatrix} C & c_{T+1}\end{bmatrix} \in \R^{d \times T+1}$. Denote with $\sigma_i(C)$ the $i$-th singular value of $C$, sorted from the largest to the smallest. Then,
\begin{gather}
    \sigma_{i+1}(\Cnew) \leq \sigma_{i}(C) \leq \sigma_{i}(\Cnew), \quad \forall i \in \{1, ..., T\}
        \label{eq:interlacing_theorem}
\end{gather}
\label{lemma:interlacing}
\end{lemma}

\begin{proof}
We have that:
\begin{gather}
    \Cnew^T\Cnew = \begin{bmatrix} C^T \\ c_{T+1}^T\end{bmatrix} \cdot \begin{bmatrix} C & c_{T+1} \end{bmatrix} = \begin{bmatrix}C^TC & C^Tc_{T+1} \\ c_{T+1}^T C & c_{T+1}^T c_{T+1} & \end{bmatrix}\enspace.
\end{gather}
Observe that $\Cnew^T\Cnew$ is a symmetric matrix and $C^TC$ is a principal submatrix. Hence, from the Eigenvalue Interlacing Theorem, we have that:
\begin{gather}
    \lambda_{i+1}(\Cnew^T\Cnew) \leq \lambda_i(C^TC) \leq \lambda_i(\Cnew^T\Cnew),
\end{gather}
where $\lambda_i(A)$ is the $i$-th eigenvalue of $A$, sorted from the largest to the smallest. To finish the proof, we note that for any matrix $A$, $\sigma_i(A) = \sqrt{\lambda_i(A^TA)}$.
\end{proof}

\section{Additional Results}
In this section, we include additional results that further support the findings of the main paper. 

\paragraph{Robustness slope} Recall Theorem \ref{th:additive_noise} of the paper.
\pagebreak
\begin{gather}
    \tikzmarknode{noisymse}{\highlight{red}{$\mse{T}{\sigma^2}$}} = \tikzmarknode{noiselessmse}{\highlight{blue}{$\mse{T}{0}$}} + \tikzmarknode{slope}{\highlight{green}{$\frac{\sum_{i=1}^{k}\sigma_i(C)^2}{T}$}} \cdot \tikzmarknode{noise}{\highlight{red}{$\sigma^2$}}\enspace.
\end{gather}

\begin{tikzpicture}[overlay,remember picture,>=stealth,nodes={align=left,inner ysep=2pt},<-]
\path (slope.south) node[anchor=north east,color=green!60] (scalep){\textbf{Robustness slope}};
\draw [color=green!87](slope.south) |- ([xshift=-0.3ex, yshift=-3ex,color=green]scalep.north west);

\path (noisymse.south) node[anchor=north east,color=red!67] (scalep){\textbf{Average MSE under noise}};
\draw [color=red!87](noisymse.south) |- ([xshift=-0.3ex, yshift=-3ex,color=red]scalep.north west);

\path (noiselessmse.north) ++ (0, 7pt) node[anchor=south west,color=blue!67] (scalep){\textbf{Average MSE without noise}};
\draw [color=blue!87](noiselessmse.north) |- ([xshift=-0.3ex,color=blue]scalep.south east);

\path (noise.south) node[anchor=north west,color=red!67] (scalep){\textbf{Noise Variance}};
\draw [color=red!87](noise.south) |- ([xshift=-0.3ex, yshift=-3ex,color=red]scalep.north east);
\end{tikzpicture}

This theoretical finding implies that the cross-over phenomenon that we observe in our experiments (at least for the linear case), stems from a lower Robustness Slope in the multitask models. Figure \ref{fig:mse_linear} shows that the MSE under noise is lower for models that are trained to do more tasks. In Figure \ref{fig:slope} of this Appendix, we show that indeed this is due to a decrease in the robustness slope. Across three different datasets, MNIST, CIFAR10, NewsGroup20, we see that increasing the number of tasks leads to a decrease in the robustness slope. We note that this does not necessarily mean a monotonic decrease in the MSE under noise. Since the total dataset size and the parameter $k$ stay the same, increasing the number of tasks usually leads to increased noiseless MSE. However, under the presense of noise, our theory predicts (and our experiments confirm) that eventually the multitask model will reach superior performance.

\begin{figure*}[!htp]
\begin{adjustbox}{width=160mm, center}
\begin{tabular}{cc}
\captionsetup{justification=centering}
\centering
\includegraphics[width=0.3\textwidth]{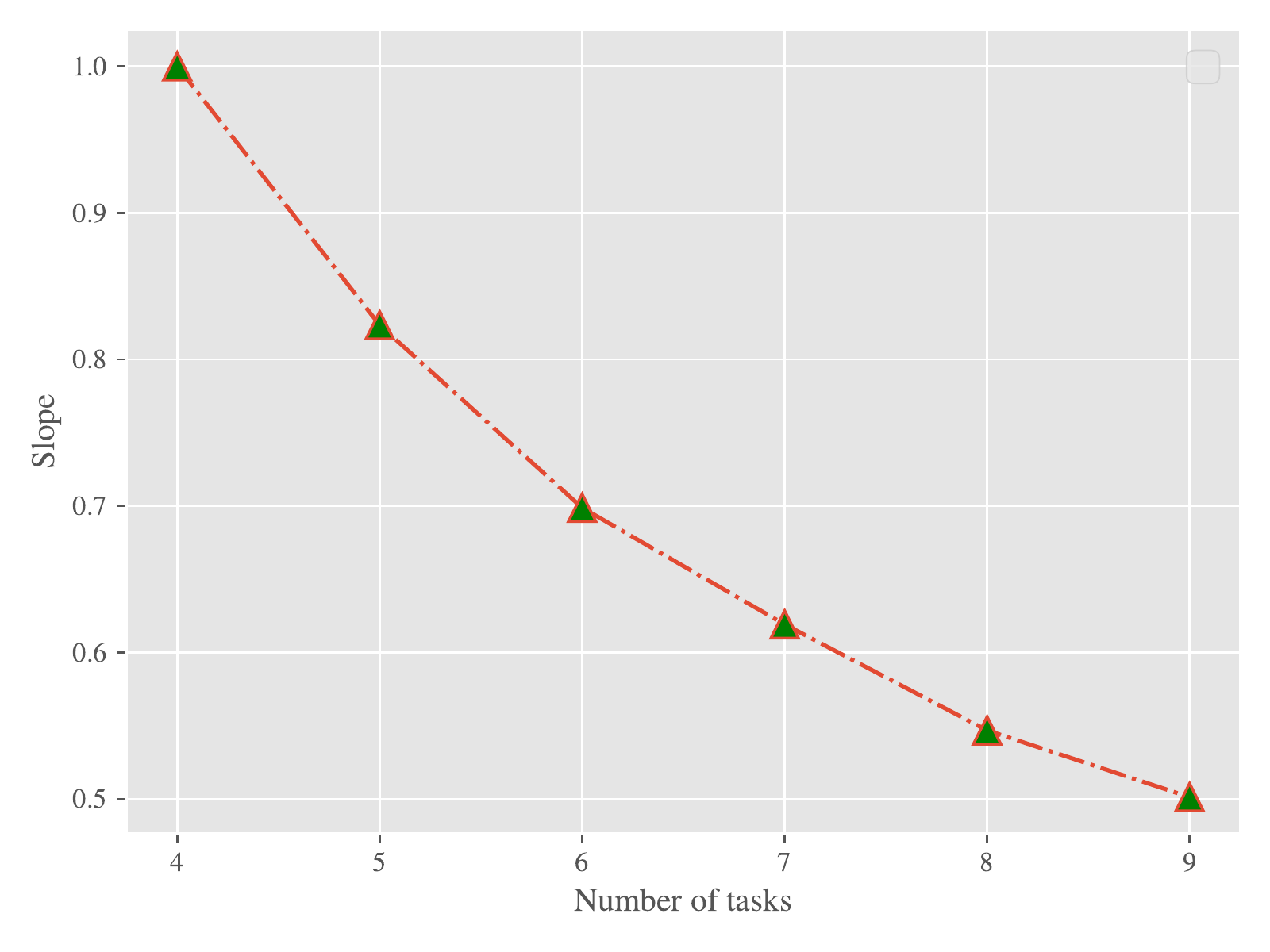} & 
\includegraphics[width=0.3\textwidth]{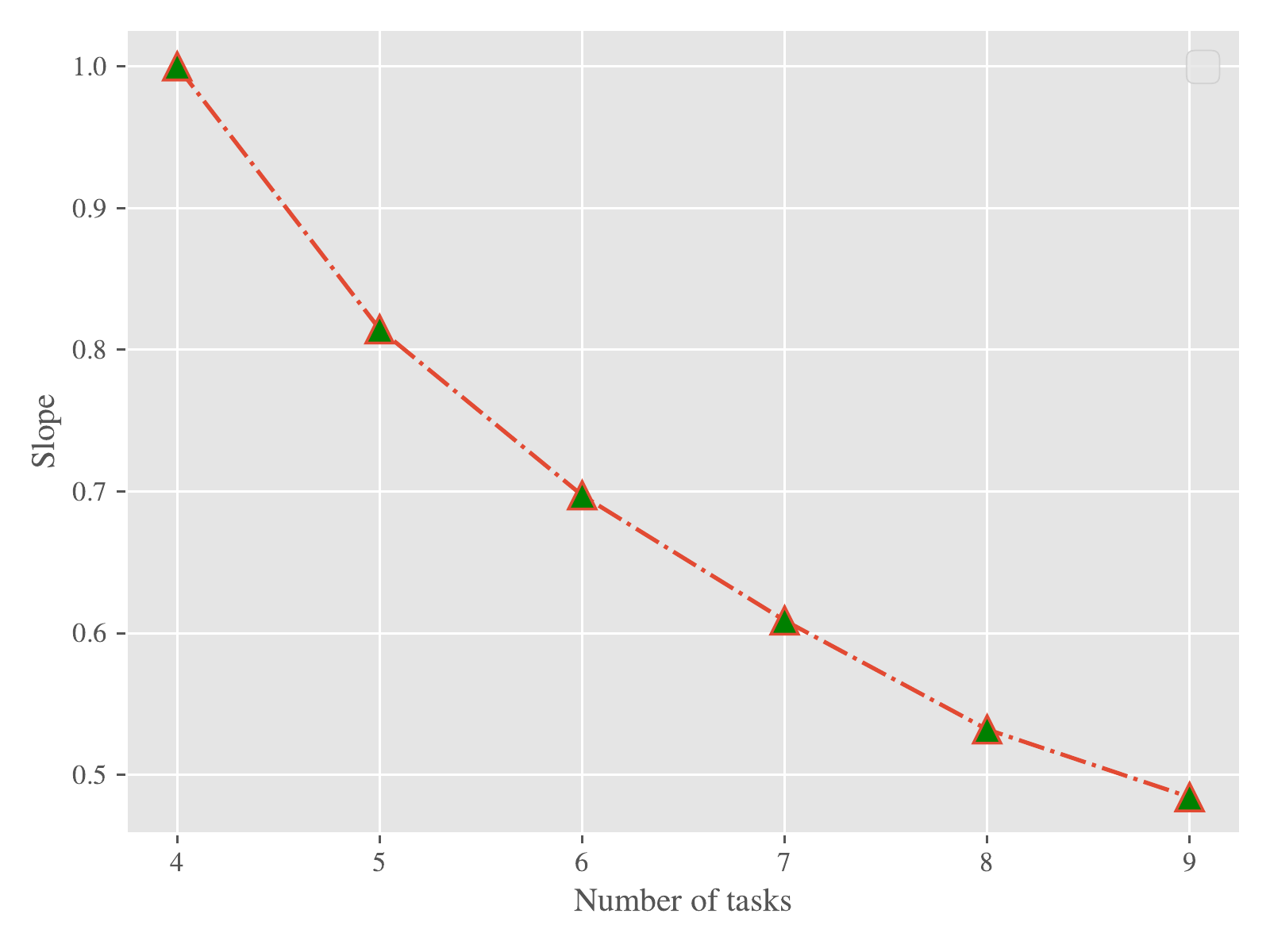} \\ 
\includegraphics[width=0.3\textwidth]{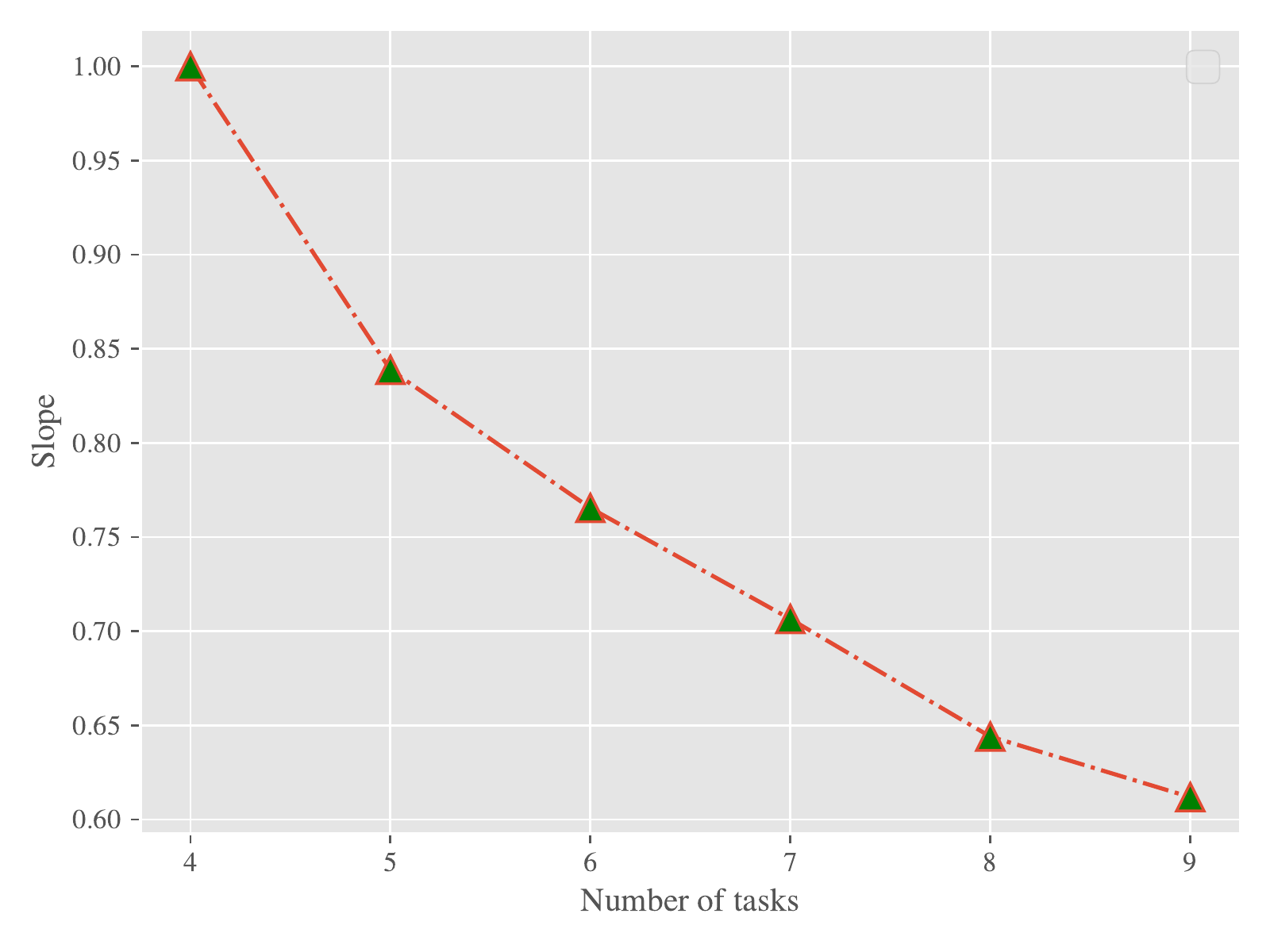}
\end{tabular}
\end{adjustbox}
\caption{Slope as a function of the number of tasks for different datasets. (1, 1): MNIST, (1, 2): CIFAR10, (2, 1): NewsGroup20. As shown, adding more tasks decreases the robustness slope which leads to an increase in robustness (see Theorem \ref{th:additive_noise}).}
\label{fig:slope}
\end{figure*}

\paragraph{Experiments on other languages} For our experiments on multilingual generative models, we decided to use Greek and English because we were looking for a linguistic pair with different morphology, syntax and phonology. This is inspired by our theory on linear models that shows that diversity in the tasks (as we have for the Gaussian task vectors) leads to a sublinear increase in the sum of the top-k singular values of the task matrix and hence an increase in robustness. For completeness, we include here experiments on a different linguistic pair, English and Spanish. English and Spanish much closer linguistically and also share the Latin alphabet, so we expect bigger transfer and smaller robustness benefit in this linguistic pair. 

We compare a monolingual English model (finetuned on English Wikipedia) with a bilingual, English and Spanish, model. The bilingual model is finetuned on a concatenation of English and Spanish Wikipedia. We make sure that the total dataset size is the same for the monolingual and the bilingual model, i.e. the bilingual model is exposed to half English data compared to the monolingual. This ensures that any benefits in terms of robustness are not coming from exposure to more data. We present results on random deletions in Figure \ref{fig:gpt2_spanish_bilingual_ppl} of the Appendix -- this Figure is similar to Figure \ref{fig:gpt2_random_deletion} of the paper, but instead of having English and Greek, we have English and Spanish. As shown in Figure \ref{fig:gpt2_spanish_bilingual_ppl}, even though the two models are starting from roughly the same perplexity, the bilingual model exhibits higher structural robustness in the presence of weight deletions. This is consistent with the results we showed across this paper and indicates that the increased robustness is not specific to the choice of the linguistic pair.

\begin{figure}[!htp]
    \centering
    \includegraphics[width=8cm]{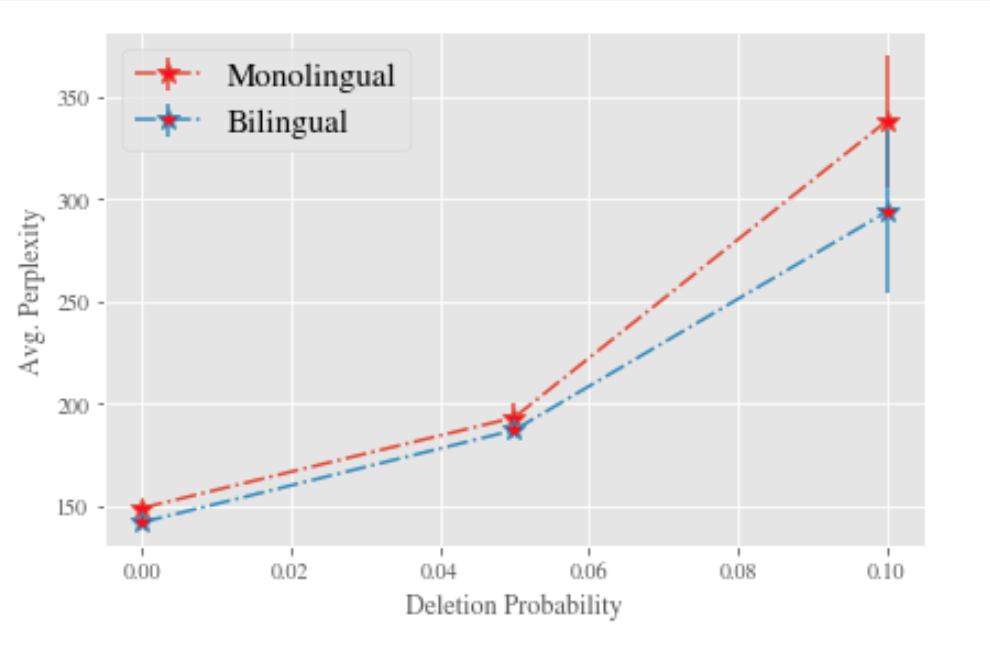}
    \caption{ Performance of monolingual (English) and bilingual (English/Spanish) GPT-2 models with the same architecture and training
dataset size. The x-axis indicates the probability of
erasing an attention weight parameter (setting to it zero). The y-axis indicates the average perplexity over 20
runs. Models have a close initial accuracy. Perplexity increases (showing lower accuracy) as weight deletion probability is increased, though bilingual model perplexity rises at a slower rate.}
   \label{fig:gpt2_spanish_bilingual_ppl}
\end{figure}

Notice that the gap in the performance is smaller compared to the one presented in Figure \ref{fig:gpt2_random_deletion}. This is aligned with our theory for linear models that predicts that the benefits of multitasking for robustness are more evident for more diverse tasks. Since English and Spanish are linguistically closer, compared to English and Greek, our intuition is that the difference in robustness is going to be smaller and this is also confirmed by this experiment. An interesting future direction is to study this robustness benefit for multiple linguistic pairs or multi-lingual models. However, this study requires massive computational resources. Similarly, it would be interesting to study how the robustness gap in bilingual models scales as the datasets scale, but this also requires training multiple pairs of GPT models to comparable accuracy, and requires computational resources that were not available to us. We hope that future research is going to shed more light into these exciting directions.

\paragraph{Experiments with different corruption mechanisms.} In the main paper, we primarily presented results with random deletions of neurons as our corruption model for the language modeling experiments. We include results for additive Gaussian noise for GPT-2 (monolingual and bilingual). We choose to present additional results with this noise model since it is the one analyzed by our theory. Table \ref{table:gaussian_noise_table} summarizes how the performance of GPT-2 (monolingual and bilingual) changes when we add different amount of noise to the weights. We evaluate this performance on downstream tasks from the GLUE paper. Figure \ref{fig:add_g} visualizes the decrease of performance as the magnitude of the noise rises for different number of tasks. The results are similar with the results presented in the main paper for random deletions. In QQP, the monolingual model performs better without perturbations. Both models decay with a close rate. The monolingual model outperforms in SST-2 with no perturbations. Both models decay with a close rate. For CoLA, the monolingual model maintains a significantly better performance regardless of the noise level. Finally, for MRPC we see that although the bilingual model shows a weaker classification accuracy with no noise, it outperforms the monolingual model for noise levels higher than 0.035. These results complement Figure \ref{fig:gpt2_magnitude_deletion} of the main paper that shows robustness of GPT-2 to additive Gaussian noise for the task of language modeling.

\begin{figure*}[htbp]
\begin{center}
\begin{subfigure}
\centering
    \includegraphics[width=0.44\linewidth]{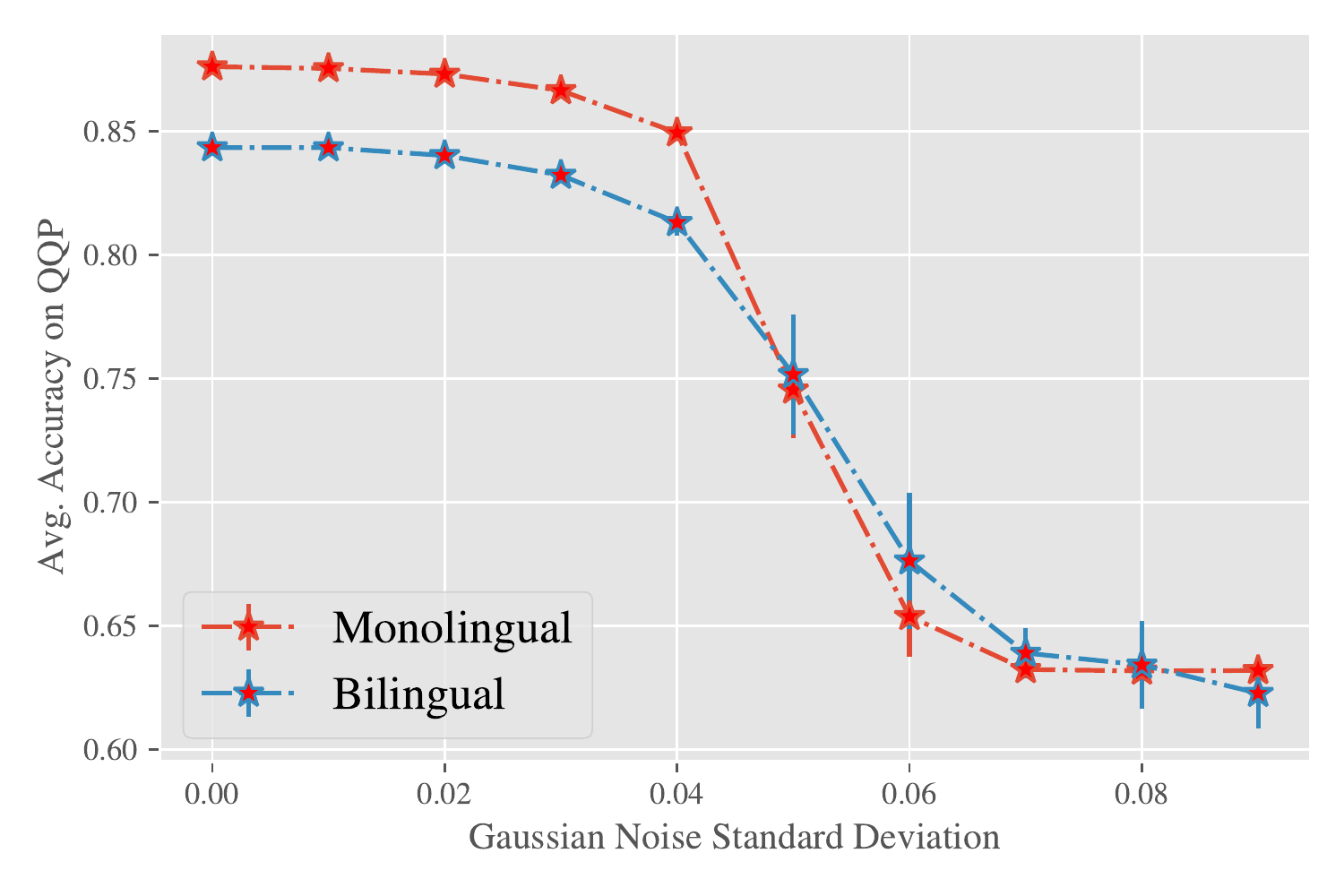}
\end{subfigure}
\begin{subfigure}
\centering
    \includegraphics[width=0.44\linewidth]{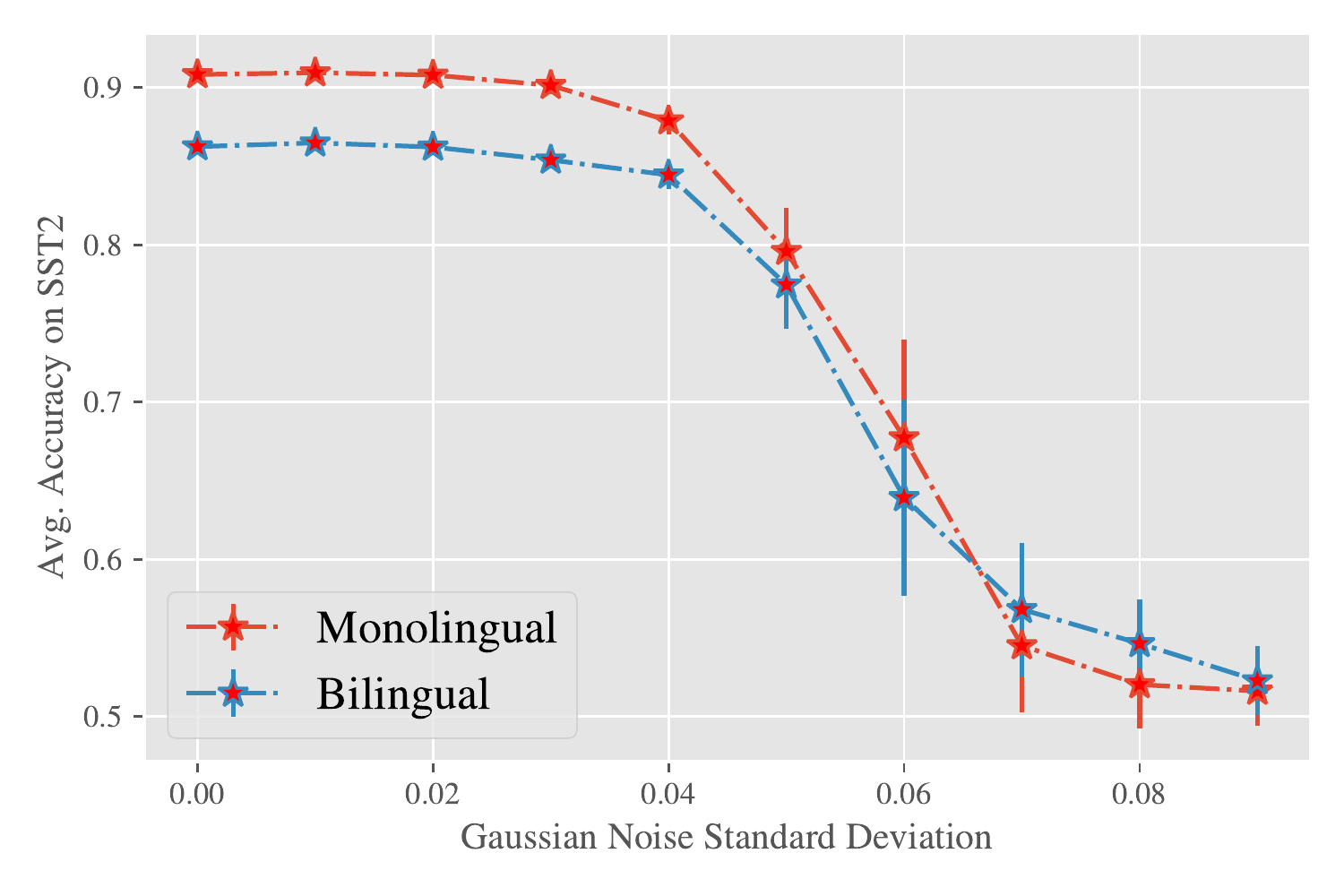}
\end{subfigure}
\begin{subfigure}
\centering
    \includegraphics[width=0.44\linewidth]{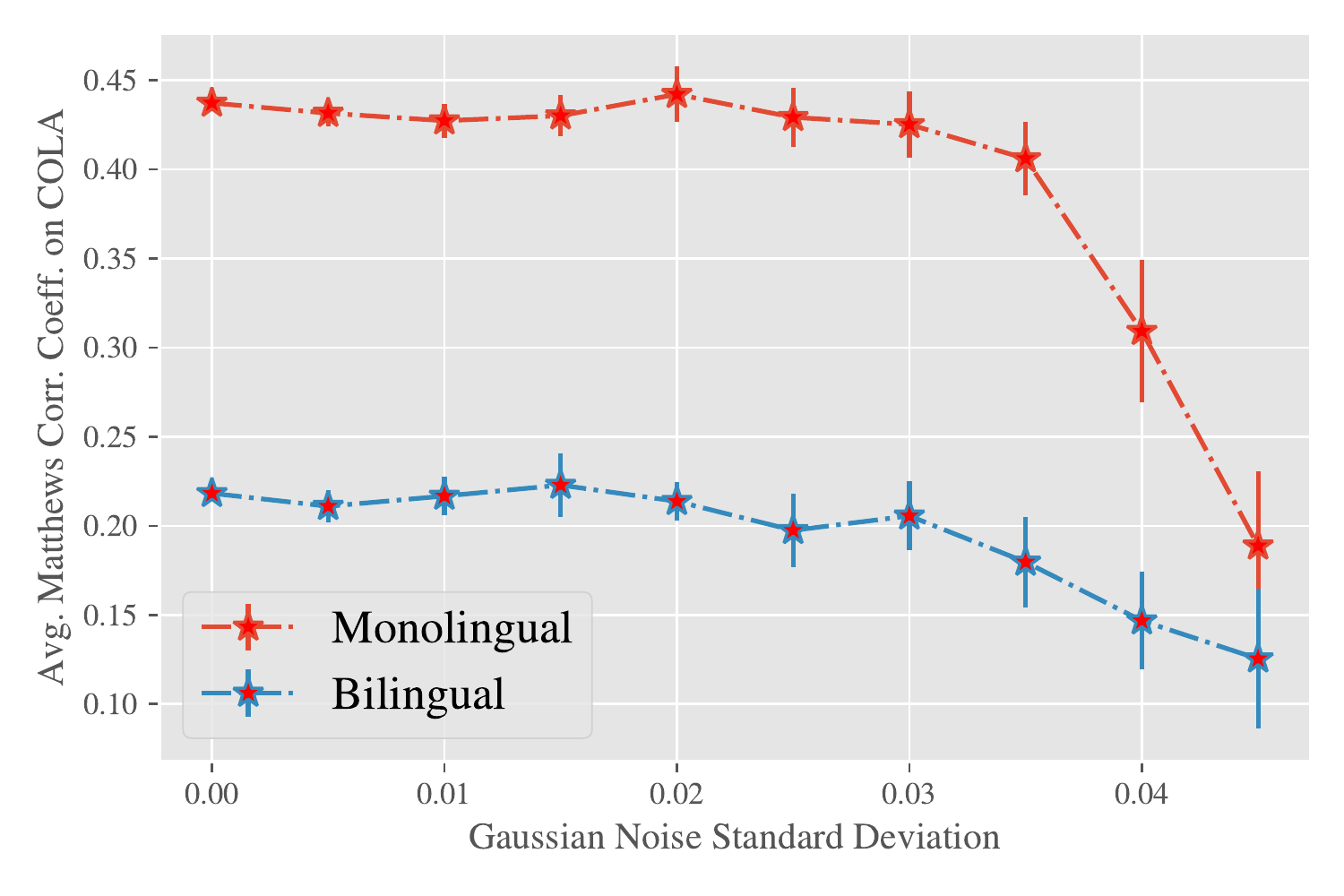}
\end{subfigure}
\begin{subfigure}
\centering
    \includegraphics[width=0.44\linewidth]{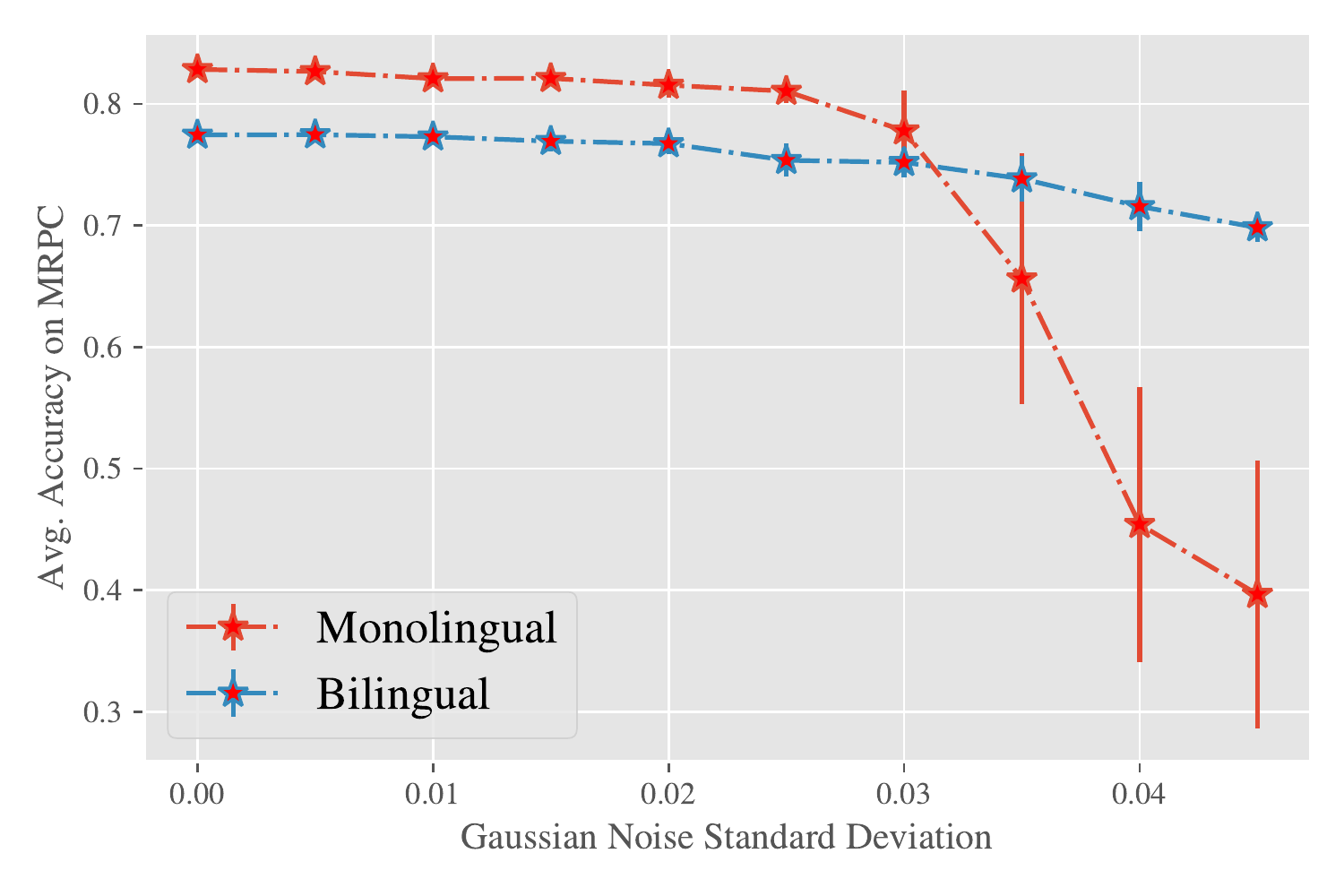}
\end{subfigure}
\end{center}
\caption{Performance comparison in GLUE tasks: QQP, SST2, CoLA, and MRPC under Gaussian noise. QQP: The monolingual model performs better without perturbations. Both models decay with a close rate. SST2: The monolingual model outperforms with no perturbations. Both models decay with a close rate. CoLA: The monolingual model maintains a significantly better performance regardless of the noise level. MRPC: Although the bilingual model shows a weaker classification accuracy with no noise, it outperforms the monolingual model for noise levels higher than 0.035.}
\label{fig:add_g}
\end{figure*}

\begin{table*}[!htp]
    \centering
\small
\begin{tabular}{c|cc|cc|cc|cc|cc}
Gaussian std. & \multicolumn{2}{c}{QQP}  & \multicolumn{2}{c}{SST2} & \multicolumn{2}{c}{COLA}  & \multicolumn{2}{c}{MRPC} & \multicolumn{2}{c}{RTE} \\
 & m. & b. & m. & b. & m. & b. & m. & b. & m. & b. \\
0.00 & 0.876& 0.843& 0.908 & 0.862 & 0.437& 0.218 & 0.828 & 0.774 & 0.646 & 0.595\\
0.01 & 0.875& 0.843& 0.909 & 0.864 & 0.432& 0.216 & 0.827 & 0.772 & 0.639 & 0.597\\
0.02 & 0.873& 0.840& 0.907 & 0.862 & 0.444& 0.213 & 0.816 & 0.760 & 0.641 & 0.591\\
0.03 & 0.866& 0.832& 0.901 & 0.853 & 0.440& 0.205 & 0.776 & 0.749 & 0.643 & 0.590\\
0.04 & 0.849& 0.813& 0.878 & 0.844 & 0.316& 0.146 & \color{forestgreen}0.494 & \color{forestgreen}\textbf{0.711} & 0.634 & 0.589\\
0.05 & \color{forestgreen}0.745& \color{forestgreen}\textbf{0.751}& 0.795 & 0.774 & 0.088& 0.074 & 0.326 & \textbf{0.673} & 0.622 & 0.577\\
0.06 & 0.653& \textbf{0.676}& 0.677 & 0.639 &-0.002& 0.004 & 0.316 & \textbf{0.610} & 0.602 & 0.563\\
0.07 & 0.632& \textbf{0.638}& \color{forestgreen}0.545 & \color{forestgreen}\textbf{0.568} & 0.019& 0.002 & 0.316 & \textbf{0.577} & 0.585 & 0.562\\
0.08 & 0.631& \textbf{0.634}& 0.520 & \textbf{0.546} &-0.006&-0.006 & 0.316 & \textbf{0.465} & 0.539 & 0.546\\
0.09 & 0.631& 0.622& 0.516 & \textbf{0.522} &-0.014& 0.002 & 0.316 & \textbf{0.451} & 0.536 & 0.528\\
\end{tabular}
     \caption{ Performance on GLUE when adding Gaussian noise. Columns labeled with "m" determine classification accuracy of monolingual models and columns labeled as "b" correspond to bilingual models. 
     CoLA is evaluated using Matthew's Correlation and other tasks are evaluated by accuracy.
     }
     \label{table:gaussian_noise_table}
\end{table*}

\begin{table}[!htp]
    \centering
\begin{tabular}{c}
    \includegraphics[width=8.5cm]{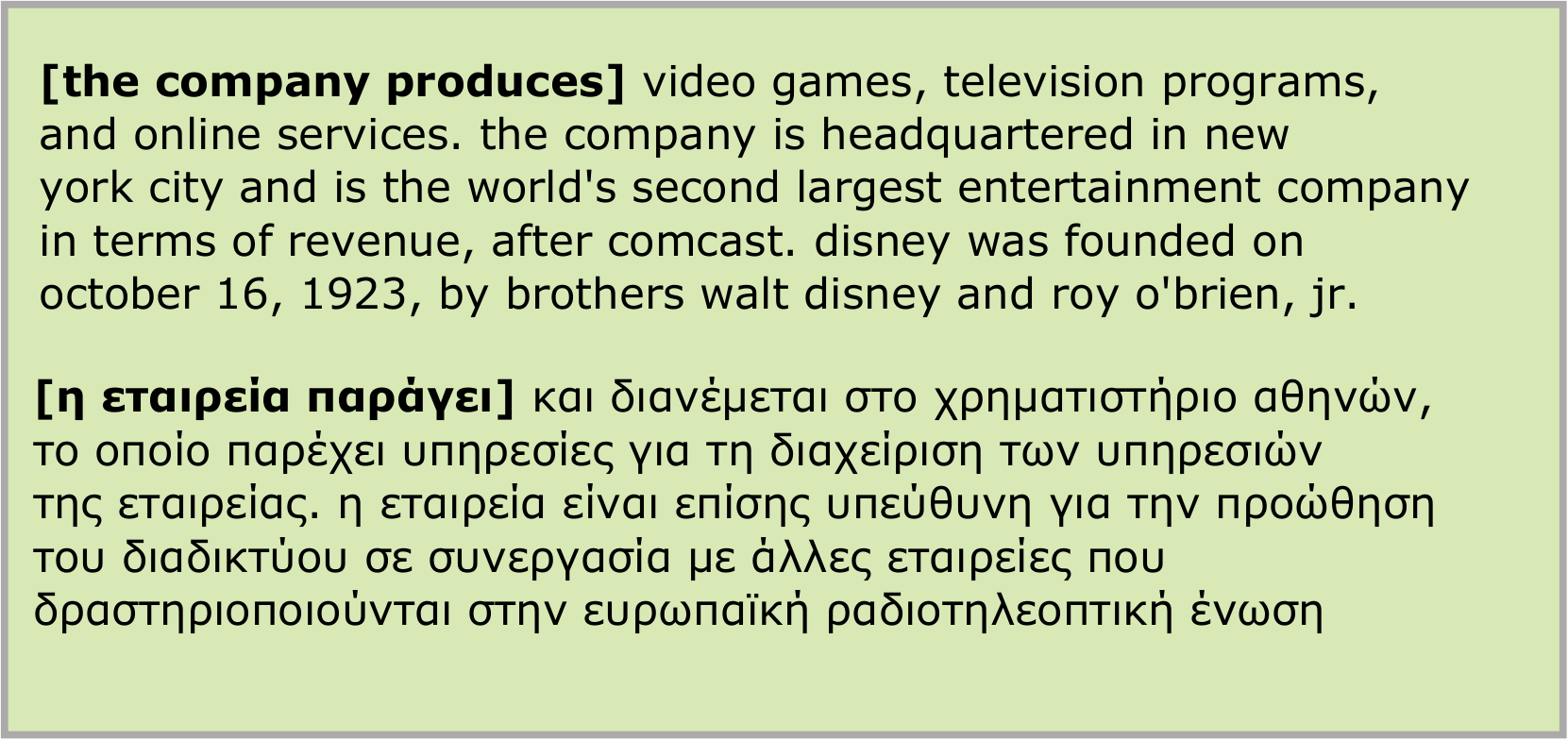}
     \end{tabular}
     \caption{Sample text generated by the bilingual GPT-2 model. Text in the brackets is the starting prompt provided for model.}
     \label{table:generated_text}
\end{table}

\section{Training Details}
Using the GPT-2 small model as our baseline, we fine-tuned a monolingual (English) model and a bilingual (English and Greek) model on Wikipedia text data. With set the vocabulary size to 50257 tokens. In both training processes, we set the initial learning rate to 3e-4 and configured a cosine learning rate scheduler with 150 warmup steps, setting AdamW optimizer weight decay to 0.01. We trained each model for eight epochs, using 4 NVIDIA Quadro RTX 5000 GPUs. Training took approximately 10 hours per epoch.

To fine-tune another bilingual model on English and Spanish data, we fine-tuned a monolingual model and a bilingual model on Wikipedia text data. With a vocabulary size of 50257 tokens, the monolingual model was fine-tuned on 800,000 English articles. The bilingual model was fined tuned on a mix of 400,000 Spanish and 400,000 English articles, using the same vocabulary size of 50257. Like the previous experiment, we set the initial learning rate to 3e-4 and configured a cosine learning rate scheduler with 150 warmup steps, setting AdamW optimizer weight decay to 0.01.

We further tuned bilingual and monolingual models for the text classification experiments using GLUE datasets. For these experiments, we used the AdamW optimizer with a learning rate of 2e-5, and epsilon at 1e-8. We used a linear scheduler with no warmup steps and trained models for more than ten epochs.

\begin{figure}[!t]
    \centering
    \includegraphics[width=8.5cm]{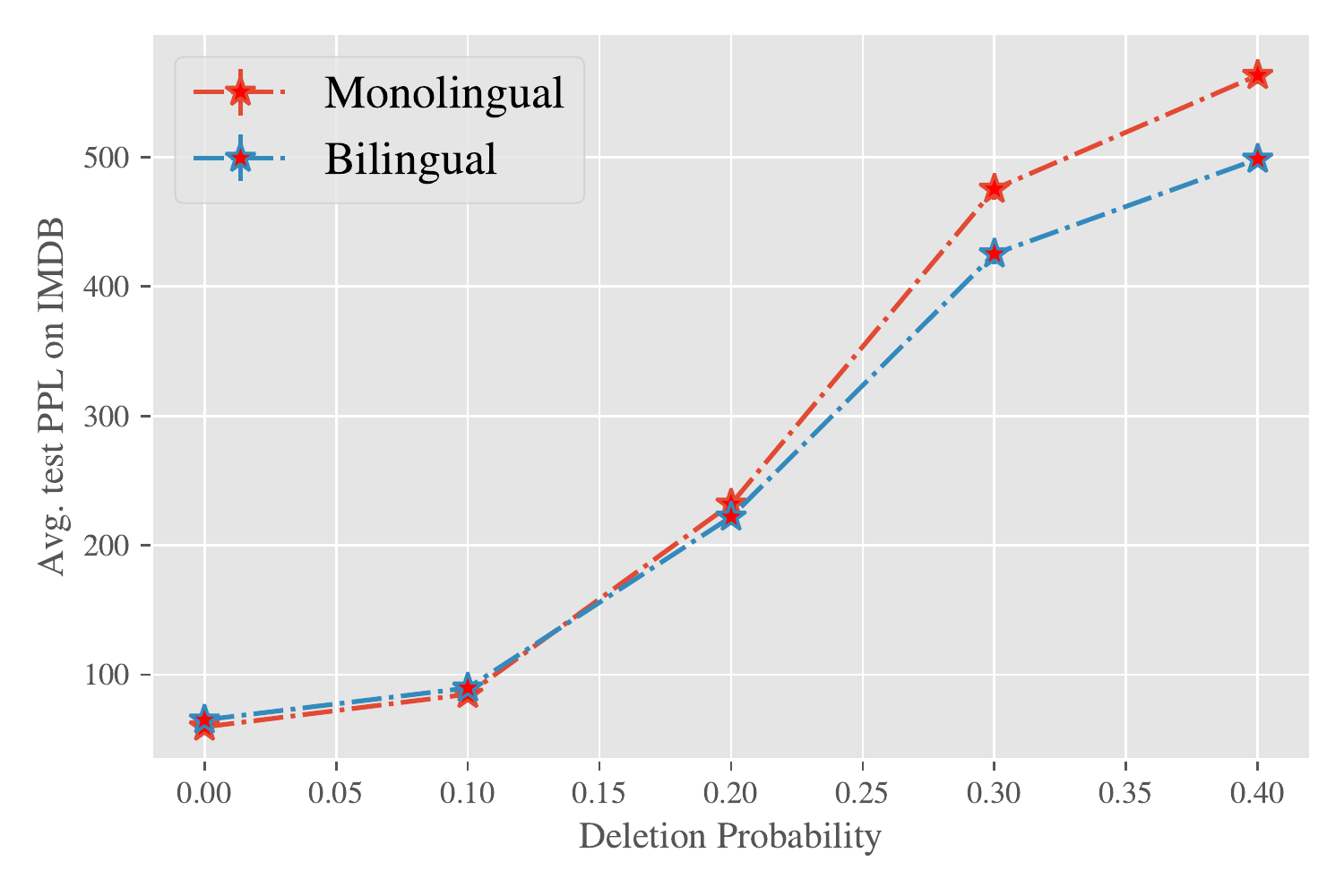}
    \caption{\small 
    Performance of monolingual and bilingual GPT-2 models with the same architecture and training
dataset size. Models are trained using the truncation length 1024. We show the performance as we randomly erase weights. The plot indicates the average perplexity over 20
runs with 95\% confidence intervals. This plot indicates that the monolingual model declines faster and performs worse in the highly damaged regime. The bilingual GPT-2 model is more robust to neuron weight erasures.}
   \label{fig:fig1correction}
\end{figure}

\begin{figure}[!t]
    \centering
    \includegraphics[width=8.5cm]{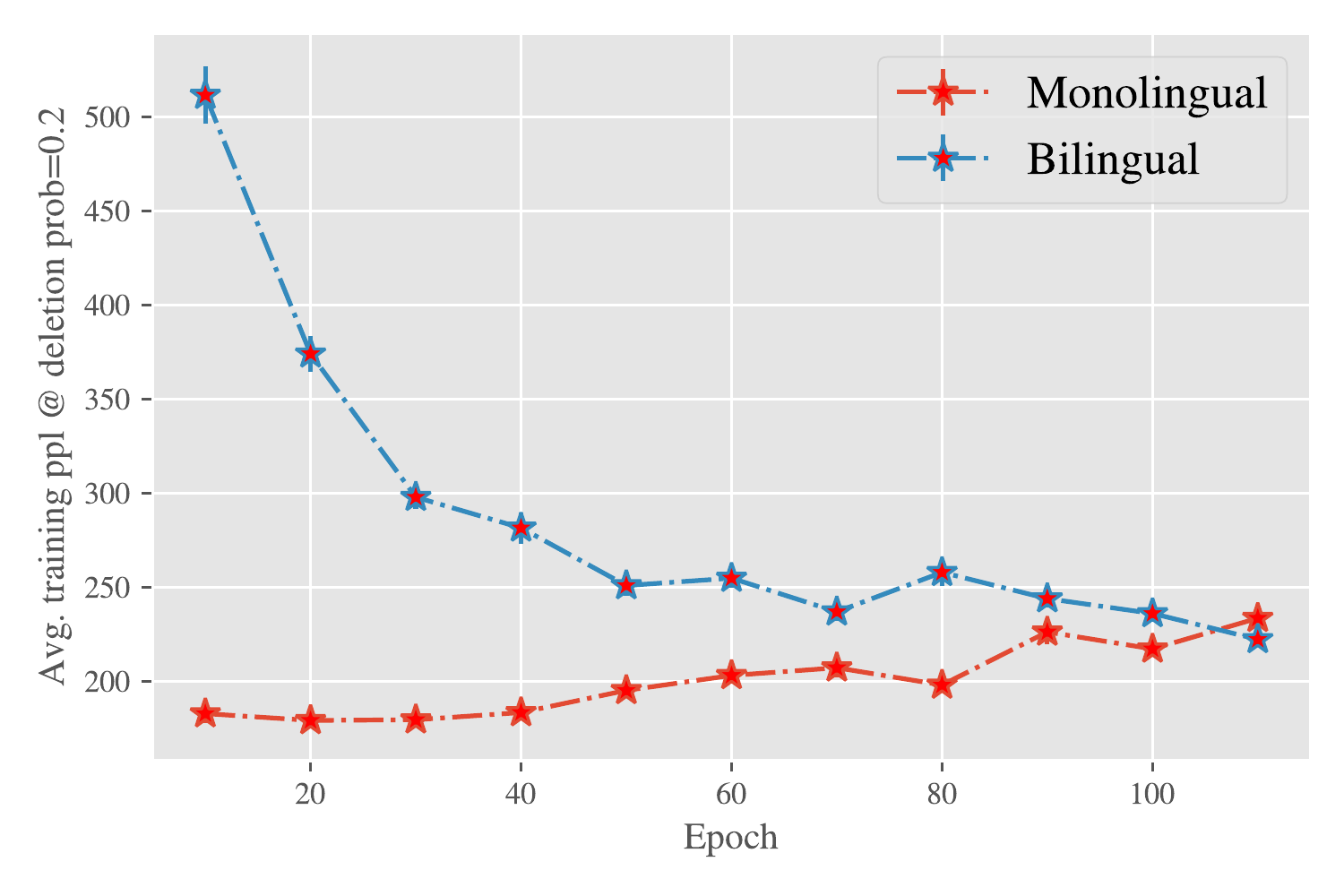}
    \caption{\small The plot shows performance of the monolingual and bilingual models during training. We plot performance at 20 percent random weight deletions to observe robustness behavior as models are trained. Notice that the bilingual model starts with very poor performance because the tokenizer has been randomized due to the recycling process. As training progresses, the bilingual model outperforms the monolingual model and becomes more robust.}
   \label{fig:varyingtraining}
\end{figure}

For the experiments on the linear representation layer, we used Adam optimizer with weight decay $1e-4$. We trained all the models with a batch size of $128$ and to a maximum of 50 epochs. To emulate multiple tasks, we selected different subsets of the classes. We experiments with having class overlaps (e.g. for MNIST one task might have been predicting 0 vs 1 and some other task predicting 1 vs 2) and without class overlaps (e.g. predicting 0 vs 1 and 2 vs 3). We noticed bigger robustness benefits when there was no class overlap something that is consistent with our theoretical analysis that implies that diversity in the tasks is needed. In terms of corruptions, we also did preliminary experiments on random deletions and we saw similar results. The interested reader might use the released code to perform other types of weights corruptions and see how this affects robustness trends. For all our experiments, we fix the representation dimension to $k=4$, which is also why we show the robustness slope from $k=4$ onwards on Figure \ref{fig:slope}. Training time of the linear experiments depends on the dataset size: it took us roughly 1 hour on CIFAR-10 and 3 hours on NewsGroup20.

While preparing the camera-ready for this paper we noticed a parameter mismatch between the training parameters of the monolingual and the bilingual model affecting some of our experiments. Specifically, the bilingual training dataset was tokenized with sentences truncated at length 128 while the monolingual dataset was tokenized using truncation length 1024. To check if this hyperparameter mismatch causes the bilingual robustness benefit, we re-trained the bilingual model using sequence truncation length 1024. This experiment is shown in Figure  \ref{fig:fig1correction} and shows similar behaviour as before. 

We performed an additional experiment to study how robustness changes during training. In Figure~\ref{fig:varyingtraining} we plot the performance of the monolingual and bilingual models as we train, but plotting performance at 20 percent random deletions, i.e. measuring also robustness. We show that as training progresses, the bilingual model becomes more robust compared to the monolingual one. 

The training datasets for the bilingual and monolingual models have the same size. The bilingual dataset is a concatenation of half Greek (from Greek Wikipedia) and half English text (from the cc\_news dataset \citep{cc_news}). Similar to previous experiments, we use the Language Model Recycling Technique \citep{gpt_recycling} to pre-trained GPT-2 (small)\citep{radford2019language}.

\section{Things that did not work}
In the early stages of the project, we attempted to train a bilingual model from scratch, instead of using the recycling technique~\cite{gpt_recycling}. The dataset for the Greek model consists of roughly $2$GB of text from Wikipedia. With such limited amount of data, we found it impossible to train a bilingual model that reaches a reasonable perplexity. Note that GPT-2 was trained on $\approx40$GB of text, i.e. on a $\approx 20\times$ bigger dataset. We found that the recycling technique~\cite{gpt_recycling} enables learning with much smaller datasets (on top of the computational benefits it offers).

\section{Limitations and Ethical Considerations}
\paragraph{Limitations} Even though the models we train can produce text of reasonable quality (e.g. see Table \ref{table:generated_text}), they do not perform on par with state-of-the-art generative networks. There are many reasons for that, e.g. we do not have the computational resources to train bigger networks and the dataset size is small. Nevertheless, the goal of this paper is not to advance the state-of-the-art in text-generation but to shed light on how multitasking is related to robustness. 

The motivation of this paper was a theory from Cognitive Science about increased robustness in bilingual speakers. We see that bilingual artificial networks are also more robust compared to monolingual models trained under the same setting. However, it is important to state that no definite extrapolations should be made to Cognitive Neuroscience without significantly much work. Our models of corruptions happening to the neural network's weights are chosen primarily for simplicity in the implementation and in the analysis. There is no evidence that brain pathologies have any resemblance to the models of corruption analyzed in this work for artificial neural networks.

Finally, our theory did not analyze the learning dynamics for approximating the task vectors. Instead, it used the SVD solution. Different choices of learning algorithms might lead to different behaviors regarding robustness. For example, for the linear case we showed that multitasking creates weight regularization. Higher explicit weight regularization (e.g. with high weight decay) might help the single task model decrease the robustness gap with the multitask networks. It would also be interesting to explore how the theory can be generalized to the non-linear case.

\paragraph{Ethical Considerations} As part of this work, we are releasing pre-trained bilingual models. Big language models can be misused in many different ways including spreading of fake news, generation of toxic speech, etc. We encourage the readers to refer to~\citet{bender2021dangers, brown2020language} for an extended discuss of the risks of releasing powerful language models. In our case, the released models are not nearly as big or powerful as state-of-the-art networks such as GPT-3. For all our experiments, we are using the small version of GPT-2 and the main objective is to see how learning multiple languages affects robustness to weight corruptions. Additionally, we are not training these models from scratch, but we are using the recycling technique proposed in~\citet{gpt_recycling}, hence the environmental cost of the training is much smaller.

\section{Code and License}
We open-source our code and pre-trained models to encourage more related research: \href{https://github.com/giannisdaras/multilingual_robustness}{https://github.com/giannisdaras/multilingual\_robustness}. The code us released under the GNU GENERAL PUBLIC LICENCE. The interested reader should also refer to the licenses of pre-existing software we use. Please look at the \texttt{requirements.txt} file of our code to find all our dependencies. 

The code for the training of the bilingual models is written in PyTorch~\citep{pytorch} and it is based on the implementation of GPT-2 found in the \texttt{transformers}~\citep{wolf2019huggingfaces} library. The code for the linear experiments is written in JAX~\cite{jax2018github}.

We expect that the release of bilingual and monolingual models trained on identical conditions will motivate further research in this area by cognitive scientists doing computational research. The main motivation for this paper was a theory from Cognitive Science regarding increased Cognitive Reserve in bilingual people. We expect that there could be many more interesting directions in Cognitive Science that can be studied from a computational perspective and we hope that the release of bilingual models will contribute towards this goal.

\end{document}